\def\isarxivversion{1} 

\ifdefined\isarxivversion
\documentclass[11pt]{article}
\else
\documentclass{article}
\fi
\ifdefined\isarxivversion

\else
\usepackage{icml2020}
\fi

\usepackage{authblk}
\usepackage{enumitem}
\usepackage{textgreek}
\usepackage{units}
\usepackage{amsmath}
\usepackage{amsthm}
\usepackage{amssymb}
\usepackage{algorithmic}
\usepackage{algorithm}
\usepackage{subfig}
\usepackage{color}
\usepackage[english]{babel}
\usepackage{graphicx}
\usepackage{grffile}
\usepackage{natbib}
\usepackage{wrapfig,epsfig}
\usepackage{epstopdf}
\usepackage{url}
\usepackage{color}
\usepackage{epstopdf}
\usepackage[T1]{fontenc}
\usepackage{bbm}
\usepackage{comment}
\usepackage{dsfont}

\usepackage{tikz}
\usepackage{hyperref}  
\hypersetup{colorlinks=true,citecolor=blue,linkcolor=blue} 
\usetikzlibrary{arrows}
\ifdefined\isarxivversion
\usepackage[margin=1in]{geometry}
\else 

\fi
\graphicspath{{./figs/}}
\usepackage{mathtools}

\newtheorem*{remark}{Remark}
\newtheorem{theorem}{Theorem}[section]
\newtheorem{lemma}[theorem]{Lemma}
\newtheorem{definition}[theorem]{Definition}

\newtheorem{proposition}[theorem]{Proposition}
\newtheorem{corollary}[theorem]{Corollary}

\newtheorem{claim}[theorem]{Claim}

\newcommand{\inner}[2]{\langle #1,#2 \rangle}

\newcommand{\wh}{\widehat}
\newcommand{\wt}{\widetilde}

\newcommand{\N}{\mathcal{N}}
\newcommand{\R}{\mathbb{R}}

\renewcommand{\d}{\mathrm{d}}

\renewcommand{\varepsilon}{\epsilon}
\renewcommand{\tilde}{\wt}
\renewcommand{\hat}{\wh}

\DeclareMathOperator*{\argmax}{argmax}
\DeclareMathOperator{\ind}{\mathds{1}}  
\DeclareMathOperator*{\E}{{\mathbb{E}}}

\DeclareMathOperator{\poly}{poly}

\DeclareMathOperator{\step}{step}
\DeclareMathOperator{\sgn}{sgn}

\definecolor{b2}{RGB}{51,153,255}
\definecolor{mygreen}{RGB}{80,180,0}
\definecolor{mycy2}{RGB}{255,51,255}

\makeatletter
\newcommand*{\RN}[1]{\expandafter\@slowromancap\romannumeral #1@}
\makeatother

\begin{document}


\ifdefined\isarxivversion

\title{Over-parameterized Adversarial Training: An Analysis Overcoming the Curse of Dimensionality}
\author[1]{Yi Zhang*}
\author[1]{Orestis Plevrakis\thanks{Equal contribution}}
\author[2]{Simon S. Du}
\author[1]{Xingguo Li}
\author[2]{Zhao Song}
\author[1,2]{Sanjeev Arora}

\affil[1]{Princeton University, Computer Science Department \authorcr
  \{\tt y.zhang, orestisp, xingguol, arora\}@cs.princeton.edu}
\affil[2]{Institute for Advanced Study\authorcr
  \tt ssdu, zhaos@ias.edu}
\else

\icmltitlerunning{Over-parameterized Adversarial Training: An Analysis Overcoming the Curse of Dimensionality}


\twocolumn[
\icmltitle{Over-parameterized Adversarial Training: An Analysis Overcoming the \\ Curse of Dimensionality}



\icmlsetsymbol{equal}{*}

\begin{icmlauthorlist}
\icmlauthor{Aeiau Zzzz}{equal,to}
\icmlauthor{Bauiu C.~Yyyy}{equal,to,goo}
\icmlauthor{Cieua Vvvvv}{goo}
\icmlauthor{Iaesut Saoeu}{ed}
\icmlauthor{Fiuea Rrrr}{to}
\icmlauthor{Tateu H.~Yasehe}{ed,to,goo}
\icmlauthor{Aaoeu Iasoh}{goo}
\icmlauthor{Buiui Eueu}{ed}
\icmlauthor{Aeuia Zzzz}{ed}
\icmlauthor{Bieea C.~Yyyy}{to,goo}
\icmlauthor{Teoau Xxxx}{ed}
\icmlauthor{Eee Pppp}{ed}
\end{icmlauthorlist}

\icmlaffiliation{to}{Department of Computation, University of Torontoland, Torontoland, Canada}
\icmlaffiliation{goo}{Googol ShallowMind, New London, Michigan, USA}
\icmlaffiliation{ed}{School of Computation, University of Edenborrow, Edenborrow, United Kingdom}

\icmlcorrespondingauthor{Cieua Vvvvv}{c.vvvvv@googol.com}
\icmlcorrespondingauthor{Eee Pppp}{ep@eden.co.uk}

\icmlkeywords{Machine Learning, ICML}

\vskip 0.3in
]

\fi

\ifdefined\isarxivversion

\begin{titlepage}
 \maketitle
  \begin{abstract}
Adversarial training is a popular method to give neural nets robustness against adversarial perturbations. In practice adversarial training leads to low robust training loss. However, a rigorous explanation for why this happens under natural conditions is still missing. Recently a convergence theory for standard (non-adversarial)  training was developed by various groups for {\em very over-parametrized} nets. It is unclear how to extend these results to adversarial training because of the min-max objective. Recently, a first step towards this direction was made by~\cite{gclhwl19} using tools from online learning, but they require the width of the net and the running time to be \emph{exponential} in input dimension $d$, and they consider an activation function that is not used in practice.
 Our work proves convergence to low robust training loss for \emph{polynomial} width and running time, instead of exponential, under natural assumptions and with ReLU activation.
 Key element of our proof is showing that ReLU networks near initialization can approximate the step function, which may be of independent interest.

  \end{abstract}
 \thispagestyle{empty}
 \end{titlepage}

\else

  \begin{abstract}
Adversarial training is a popular method to give neural nets robustness against adversarial perturbations. In practice adversarial training leads to low robust training loss. However, a rigorous explanation for why this happens under natural conditions is still missing. Recently a convergence theory for standard (non-adversarial)  training was developed by various groups for {\em very over-parametrized} nets. It is unclear how to extend these results to adversarial training because of the min-max objective. Recently, a first step towards this direction was made by~\cite{gclhwl19} using tools from online learning, but they require the width of the net and the running time to be \emph{exponential} in input dimension $d$, and they consider an activation function that is not used in practice.
 Our work proves convergence to low robust training loss for \emph{polynomial} width and running time, instead of exponential, under natural assumptions and with ReLU activation.
 Key element of our proof is showing that ReLU networks near initialization can approximate the step function, which may be of independent interest.

  \end{abstract}

\fi


\section{Introduction}
\label{sec:intro}
Deep neural networks trained by gradient based methods tend to change their answer (incorrectly) after small adversarial perturbations in inputs~\cite{szsbegf13}. Much effort has been spent to make deep nets resistant to such perturbations 
but adversarial training with a natural min-max objective~\cite{mmstv18} stands out as one of the most effective approaches according to~\cite{cw17c,athalye2018obfuscated}. 

 One interpretation of the min-max formulation is a certain two-player game between a neural network learner and an adversary who is allowed to perturb the input within certain constraints. In each round, the adversary generates new adversarial examples against the current network, on which the learner takes a gradient step to decrease its prediction loss in response (see Algorithm~\ref{alg:adv-train}).

It is empirically observed that, when the neural network is initialized randomly, this training algorithm is efficient and computes a reasonably sized neural net that is robust on (at least) the \emph{training} examples~\citep{mmstv18}. We're interested in theoretical understanding of this phenomenon:
\emph{Why does adversarial training efficiently find a feasibly sized neural net  to fit training data robustly?}
In the last couple of years, a convergence theory has been  developed for non-adversarial training: it explains the ability of gradient descent to achieve small training loss, provided the neural nets are fairly {\em over-parametrized}. But it is quite unclear whether similar analysis can be applied to adversarial training setting where the inputs are perturbed. Furthermore, while the algorithm is reminiscent of well-studied no-regret dynamics for finding equilibria in two-player zero-sum convex/concave games~\citep{hazan2016introduction}, here the game value is training loss, and hence non-convex. Thus it is unclear if training leads to small robust training loss.

A study of such issues was initiated in~\cite{gclhwl19}.  For two-layer nets with \emph{quadratic ReLU} activation\footnote{This is the activation function $\left(ReLU(x)\right)^2$.} they were able to show that if input is in $\R^d$ then training can achieve robust loss at most $\varepsilon$ provided the  net's width is  $(1/\varepsilon)^{\Omega(d)}$ (the number of required iterations is also that large)\footnote{These bounds appear in Corollary $C.1$ in their paper.}. This is very extreme over-parametrization,  and this \emph{curse of dimensionality} is inherent to their argument. 
They left as an open problem the possibility to improve the width requirement, which is the theme of our paper.

\paragraph{Our contributions:}
Under a standard and natural assumption that training data are well-separated with respect to the magnitude of the adversarial perturbations (also verified for popular datasets in Figure~\ref{fig:separation}) we show the following: 


\begin{itemize}
    \item That there exists a two-layer ReLU neural network with width $\poly\left(d, \nicefrac{n}{\varepsilon} \right)$ near Gaussian random initialization that achieves $\varepsilon$ robust training loss.
    \item That starting from Gaussian random initialization, standard adversarial training (Algorithm~\ref{alg:adv-train}) converges to such a network in $\poly\left(d, \nicefrac{n}{\varepsilon} \right)$ iterations.
    
    \item  New result in approximation theory, specifically the existence of a good approximation to the step function by a polynomially wide two-layer ReLU network with weights close to the standard gaussian initialization.  Such approximation result may be of further use in the emerging theory of over-parameterized nets.
\end{itemize}


\paragraph{Paper structure.} This paper is organized as follows. In section~\ref{sec:related}, we give an overview of the related works. In section~\ref{sec:pre}, we present our notation, the adversarial training algorithm, the separability condition and we argue why the training examples being well-separated is a natural assumption. In section~\ref{sec:results}, we formally state our main result and in section~\ref{sec:overview} we give an overview of its proof. In section~\ref{sec:proofs} we elaborate more on the core part of the proof, which is the existence of a net close to initialization that robustly fits the training data.

\section{Related Works}
\label{sec:related}
\paragraph{Adversarial examples and defense.}
The seminal paper~\cite{szsbegf13} discovered the existence of adversarial examples. Since its discovery, numerous defense methods have been proposed to make neural nets robust to perturbations constrained in a ball with respect to a certain norm (e.g. $\ell_2$, $\ell_\infty$). These methods span an extremely wide spectrum including certification~\citep{raghunathan2018certified, wong2017provable}, input transformation~\citep{buckman2018thermometer,guo2017countering}, randomization~\citep{xie2017mitigating}, adversarial training~\cite{mmstv18}, etc. Recent studies on evaluating the effectiveness of the aforementioned defenses by~\cite{cw17c,athalye2018obfuscated} reveals that adversarial training dominates the others. One empirical observation made in~\cite{mmstv18} is that adversarial training can always make wide nets achieve small robust training loss.

\paragraph{Convergence via over-parameterization.} Recently, there has been a tremendous progress in understanding the "small training loss" phenomenon in standard (non-adversarial) training~\citep{ll18,dzps19,als19a,als19b,dllwz19,adhlw19,adhlsw19,sy19,zou2018stochastic, oymak2019towards}. A convergence theory has been developed to show that, when randomly initialized, gradient descent and stochastic gradient descent converge to small training loss in polynomially many iterations when the network has polynomial width in terms of the number of training examples. 
These papers studied over-paramterized neural networks in the neural tangent kernel (NTK) regime~\citep{jgh18}.

\paragraph{Convergence of adversarial training.}
There is a growing interest in analyzing convergence properties of adversarial training.~\cite{gclhwl19} made a first attempt towards extending the aforementioned results in standard training to adversarial training. 
Like previous works on the convergence of (non-adversasrial) gradient descent for over-parameterized neural networks, this work also considered the NTK regime.
First of all, they prove that adversarial training with an artificial projection step always finds a multi-layer ReLU net that is $\epsilon$-optimal within the neighborhood near initialization, but the optimal robust loss could be large. Secondly, for two-layer quadratic ReLU net, they managed to prove that small adversarial loss will be achieved, but crucially the required width and running time are $(1/\varepsilon)^{\Omega(d)}$.
Their argument suffers \emph{the curse of dimensionality}, because it relies on the universality of the induced Reproducing Kernel Hilbert Space (induced by NTK) followed by a random feature approximation. In contrast, we take a closer look on how to approximate a robust classifier with ReLU networks near their initialization using techniques from polynomial approximation and manage to overcome this problem. In addition, our convergence analysis applies to ReLU activated nets without additional projection steps.


\paragraph{Polynomial approximation.} 
A key technique in our proof is a polynomial approximation to the step function on interval $[-1,-\eta]\cup[\eta, 1]$ which has been an important subject~\citep{allen2017faster, frostig2016principal,eremenko2006uniform}. For $\epsilon$-uniform approximation, ~\cite{frostig2016principal} constructed a polynomial with degree $\tilde{\Theta}\left(\nicefrac{1}{\eta^2}\right)$ and further proved the existence of a $\tilde{\Theta}\left(\nicefrac{1}{\eta}\right)$-degree polynomial\footnote{$\tilde{\Theta}(\cdot)$ excludes logarithmic factors.} but without algorithmic construction, which was done by~\cite{allen2017faster}. Interestingly, a nearly matching lower bound on the degree had been shown by~\cite{eremenko2006uniform} much prior to these constructions.
\section{Preliminaries}
\label{sec:pre}

\subsection{Notations}
For a vector $x$, we use $\| x\|_p$ to denote its $\ell_p$ norm, and we are mostly concerned with $p=1,2,$ or $\infty$ in this paper.

For a matrix $W \in \R^{d \times m}$, we use $W^\top$ to denote the transpose of $W$, we use $\| W \|_F$, $\|W\|_1$ and $\|W\|$ to denote its Frobenius norm, entry-wise $\ell_1$ norm, and spectral norm respectively. We define $\| W \|_{2,\infty} = \max_{j \in [d]} \| W_{j} \|_2$, and $\| W \|_{2,1} = \sum_{j=1}^d \| W_{j} \|_2$, where $W_{j}$ is the $j$-th column of $W$, for each $j \in [m]$. 
We use ${\cal N}(\mu,\Sigma)$ to denote Gaussian distribution with mean $\mu$ and covariance $\Sigma$. We denote by $\sigma(\cdot)$ the ReLU function $\sigma(z)=\max\{z,0\}$ and by $\ind\{E\}$ the indicator function for an event $E$. 
\subsection{Two-layer ReLU network}
We consider a two-layer ReLU activated neural network with $m$ neurons in the hidden layer: 
\begin{align}
f(x)= \sum_{r=1}^m a_r \sigma\left(\langle W_r,x \rangle+b_r\right)
\end{align}
where $W = (W_1,\dots, W_m)\in \mathbb{R}^{d\times m}$ is the hidden weight matrix, $b=(b_1,\dots,b_m)\in \mathbb{R}^m$ is
the bias vector, and $a=(a_1,\dots,a_m)\in \mathbb{R}^m$ is the output weight vector. We use $\mathcal{F}$ to denote this function class. During adversarial training, we only update $W$ and keep $a$ and $b$ at initialization values. For this reason, we write the network as $f_W(x  )$.

 We have $n$ training data $\mathcal{S} =$ $\{(x_1,y_1),\dots,(x_n,y_n)\}\subseteq \mathbb{R}^d\times \mathbb{R}$. We make some standard assumptions about the training set. Without loss of generality, we assume that for all $i\in[n]$, $\|x_i\|_2=1$ and the last coordinate $x_{i,d}=1/2$ \footnote{$1/2$ can be padded to the last coordinate, $\|x_i\|_2=1$ can always be ensured from $\|x_i\|_2\leq 1$ by padding $\sqrt{1-\|x\|_2^2}$.}. For this reason, we define the set ${\cal X} : = \{ x \in \R^d : \| x \|_2 =1 ,\ x_{d} = 1/2 \}$. We also assume for simplicity that for all $i\in[n]$, $|y_i|\leq 1$.

The initialization of $a,W,b$ is $a^{(0)},W^{(0)},b^{(0)}$.
\begin{itemize}
    \item The entries of $W^{(0)}$ and $b^{(0)}$ are iid random Gaussians from $\mathcal{N}(0,\frac{1}{m})$.
    \item The entries of $a^{(0)}$ are iid with distribution $unif\left(\left\{-\frac{1}{m^{1/3}},+\frac{1}{m^{1/3}}\right\}\right)$. \footnote{The choice of $m^{1/3}$ at the denominator is inessential. For technical reasons we need the distribution to be $unif\left(\left\{-\frac{1}{m^{c}},+\frac{1}{m^{c}}\right\}\right)$ for some $\Omega(1)\leq c\leq 1/3$.}
\end{itemize}
\subsection{Adversary and robust loss}
To evaluate the neural nets, we consider a loss function of the following type. 
\begin{definition}[Lipschitz convex regression loss]
A loss function $\ell: \R\times \R \rightarrow \R$ is a Lipschitz convex regression loss if it satisfies the following properties: convex in the first argument, non-negative, $1-$Lipshcitz and for all $ y\in\R, ~\ell(y, y) = 0$.
\end{definition}

We remark the choice of loss is for simplicity of technical presentation, following the convention in previous works~\cite{gclhwl19, all19}.

For a vector $z \in \R^d$ and $\rho>0$, let ${\cal B}_2(z,\rho) := \{ x \in \R^d : \| x - z \|_2 \leq \rho \}\cap \cal{X}$.
Now we define the adversarial model studied in this paper.
\begin{definition}[$\rho$-Bounded adversary]
An adversary $\mathcal{A}: \mathcal{X}\times \R \times \mathcal{F} \rightarrow \mathcal{X}$ is $\rho$-bounded for $\rho>0$ if they satisfy 
\begin{align*}
    \mathcal{A}(x, y, f)\in\mathcal{B}_2(x,\rho)
\end{align*}
We use $\mathcal{A}^{*}$ to denote the \textbf{worst-case} $\rho$-bounded adversary for loss function $\ell$, which is defined as
\begin{align*}
    \mathcal{A}^*(x, y, f):=\argmax_{\tilde{x}\in\mathcal{B}_2(x,\rho)} \ell(f(\tilde{x}), y)
\end{align*}
With a slight abuse of notation, we use $\mathcal{A}(S,f):=\{(\mathcal{A}(x_i, y_i, f), y_i)\}_{i=1}^n$ to denote the adversarial dataset generated by $\mathcal{A}$ against a given neural net $f$.
\end{definition}

We now define the robust loss of $f$ in terms of its prediction loss on the examples generated by an adversary.

\begin{definition}[Training loss and its robust version]
Given a training set $S$ of $n$ examples, the standard training loss of a neural net $f$ is defined as $\mathcal{L}(f, S) := \frac{1}{n} \sum_{i=1}^n \ell\left(f(x_i), y_i\right)$. Against a $\rho$-bounded adversary $\mathcal{A}$, we define the robust training loss w.r.t. $\mathcal{A}$ as
\begin{align*}
{\cal L}_\mathcal{A}(f) := {\cal L}(f, \mathcal{A}(S, f)) = \frac{1}{n} \sum_{i=1}^n \ell\left(f(\mathcal{A}(x_i, y_i, f)), y_i\right) 
\end{align*}
Furthermore, we define analogously the \textbf{worst-case} robust training loss as
\begin{align*}
{\cal L}_{\mathcal{A}^*}(f) := {\cal L}(f, \mathcal{A}^*(S, f)) = \frac{1}{n} \sum_{i=1}^n \max_{\tilde{x_i}\in\mathcal{B}_2(x_i,\rho)}\ell\left(f(\tilde{x_i}), y_i\right) 
\end{align*}
\end{definition}
\subsection{Well-separated training sets}
Training set being well-separated is a standard assumption in over-parametrization literature. Here we require a slightly stronger notion since we are dealing with adversarial perturbations. 

\begin{definition}[$\gamma$-separability]
We say a training set $S$ is $\gamma$-separable with respect to a $\rho$-bounded adversary, if for all $i\not=j\in[n]$, $~\|x_i - x_j\|_2 \geq\delta$ and $\gamma\leq \delta(\delta-2\rho)$.

\end{definition}

\begin{figure}[ht]
\centering
\includegraphics[width=8cm]{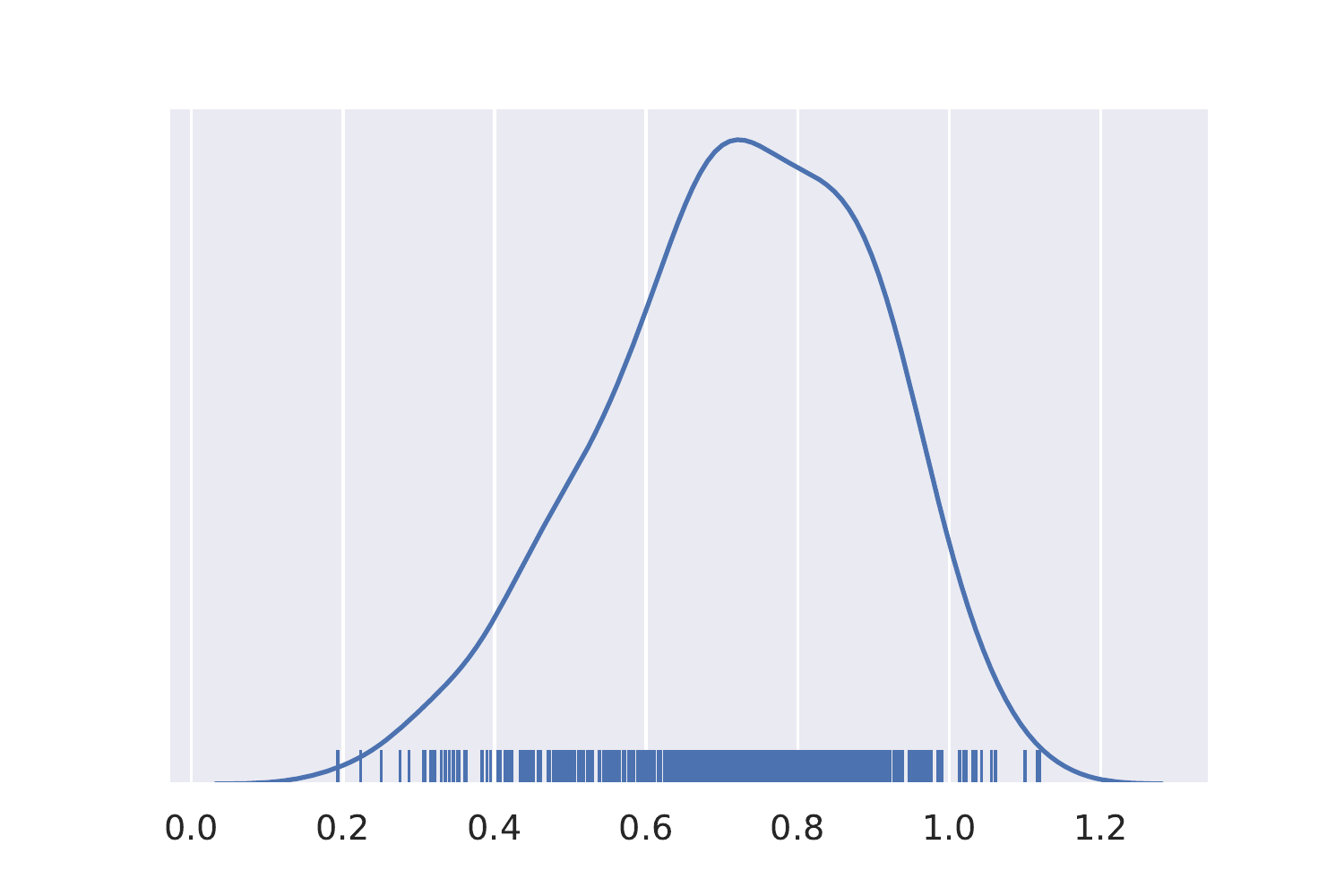}
\caption{Distribution of $\delta_i$'s of randomly sampled $500$ points in CIFAR-10 training set, where $\delta_i$ is the smallest $\ell_2$ distance between data point $x_i$ and any other point in the training set.}
\label{fig:separation}
\end{figure}
Our results imply that the required width is polynomial for $\Omega(1)$-separable training sets.
To see why this is a reasonable assumption, $\delta\approx\sqrt{3/2}$ if $x$'s are drawn from the uniform distribution on $\cal{X}$ and $d$ is large, while $\rho$ is usually at most $1/20$ in practice~\citep{guo2017countering}. In Figure~\ref{fig:separation} we show that on CIFAR-10, other than probably a very small fraction of examples, all the others do not have too small minimum distance from any example.\footnote{One can always exclude this small fraction from the training set and then suffer this fraction at the final robust 0-1 loss.}

\subsection{Adversarial training algorithm}
 The adversarial training of a neural net $f_W$ against an adversary $\mathcal{A}$ can be captured as the following intertwining dynamics. 
\begin{algorithm}[t]
    \caption{Adversarial training}
    \label{alg:adv-train}
	\begin{algorithmic}
		\REQUIRE Training set $S=\{(x_1,y_1),\dots,(x_n,y_n) \}$, Adversary $\mathcal{A}$, learning rate $\eta$, initialization $a^{(0)},W^{(0)},b^{(0)}$.
	    \FOR{ $t = 0$ to $T-1$}
	        \STATE $S^{(t)} := \emptyset$
	        \FOR{$i=1$ to $n$}
	            \STATE $\tilde{x}_i^{(t)} = \mathcal{A}(x_i, y_i, f_{W^{(t)}})$
	            \STATE $S^{(t)} = S^{(t)} \cup (\tilde{x}_i^{(t)}, y_i)$
	        \ENDFOR
	        
	        \STATE $W^{(t+1)} = W^{(t)} - \eta\cdot~\nabla_W \mathcal{L}(f_{W^{(t)}}, S^{(t)})$.
	    \ENDFOR
	    \ENSURE $\{W^{(t)}\}_{t=1}^T$
	\end{algorithmic}
\end{algorithm}
In the inner loop, the adversary generates adversarial examples against the current neural net. In the outer loop, a gradient descent step is taken on the neural net's parameter to decrease its prediction loss on the fresh adversarial examples.

\begin{remark}
The gradient computation 
$\nabla_W \mathcal{L}(f_{W^{(t)}}, S^{(t)})$ is undertaken pretending as if $S^{(t)}$ was independent from $W^{(t)}$, i.e., without differentiating through $\mathcal{A}$.
\end{remark}



\section{Main Result}\label{sec:results}

We now formally present our main theorem.

\begin{theorem}\label{thm:main}
Suppose that the training set $\mathcal{S}$ is $\gamma$-separable, for some $\gamma>0$. Then, for all $\epsilon\in(0,1)$, there exist 
\begin{align*}
     M_0=\poly\left(d,\left(\frac{n}{\epsilon}\right)^{1/\gamma}\right)   ~   \text{and} ~~ R=\poly\left(\left(\frac{n}{\epsilon}\right)^{1/\gamma}\right)
 \end{align*} 
 such that for every $m\geq M_0$, with probability at least $1- \exp\left( -\Omega\left(m^{1/3}\right) \right)$ over the choice of $a^{(0)},W^{(0)},b^{(0)}$, if we run adversarial training~\ref{alg:adv-train} with hyper-parameters
\begin{align*}
T=\Theta( \epsilon^{-2}  R^2 )\ \text{and} \ \eta=\Theta(\epsilon m^{-1/3} )
\end{align*}
then the output weights $\left(W^{(t)}\right)_{t=1}^T$ satisfy
\begin{align*}
    \min_{t\in[T]} \mathcal{L}_{\mathcal{A}}\left(f_{W^{(t)}}\right) \leq \epsilon
\end{align*}
\end{theorem}


\section{Proof Overview}\label{sec:overview}
\subsection*{Pseudo-network}
The key property used in all recent papers that analyze gradient descent for over-parameterized neural nets is that if a network $f_W(x )=\sum_{r=1}^m a_r^{(0)} \sigma\left(\langle W_r,x \rangle +b_r^{(0)}\right)$ is very over-parameterized and its weights are close to initialization, then it is well-approximated by its corresponding \emph{pseudo-network}: 
\begin{align*}
g_W(x )=\sum_{r=1}^m a_r^{(0)} \langle W_r-W_r^{(0)},x \rangle \ind \left\{\langle W_r^{(0)},x \rangle +b_r^{(0)} \geq 0\right\}
\end{align*}
However, the approximation result used for standard training is insufficient for our purposes, because here we deal with adversarial perturbations and in order to argue that during adversarial training the network behaves essentially as a pseudo-network, we need an approximation guarantee that holds \emph{uniformly} over all $\mathcal{X}$. More specifically, in these works, it is proven that for any fixed input $x$, with probability at least $1-e^{-\poly(\log m)}$, for $W$ close to the initialization, $|f_W(x)-g_W(x)|$ is small. But, with this probability bound, in order to argue that $\sup_{x\in \cal{X}}|f_W(x)-g_W(x)|$ is small via $\epsilon$-net arguments, one needs $m\geq \exp(\Omega(d))$. In this work, we show that the guarantee for fixed $x$ actually holds with much higher probability: $1-\exp(-\Omega( m^{1/3}))$. The fact that this approximation fails with exponentially small probability, enables us to take a union bound over a very fine-grained $\frac{1}{\poly(m)}$-net of $\mathcal{X}$, and even though it has cardinality $\exp(O(d\log m))$, the width $m$ we need to control the overall probability is still polynomial in $d$. The final step is to bound the stability of $f$ and $g$ under small perturbations, even though $g$ is not Lipschitz continuous.
\begin{theorem}\label{thm:real_approximates_pseudo}
  Let $R\geq 1$. For all $m \geq \poly( d)$, with probability at least $1-\exp(-\Omega( m^{1/3}) )$ over the choice of $a^{(0)},W^{(0)},b^{(0)}$, for all $W\in\mathbb{R}^{d\times m}$ such that $\|W-W^{(0)}\|_{2,\infty}\leq \frac{R}{m^{2/3}}$,
   \begin{align*}
  \sup_{x\in \mathcal{X}} \left|f_W(x)-g_W(x)\right| \leq O\left(\frac{R^2}{m^{1/6}}\right)
   \end{align*}
\end{theorem}
We give the proof of Theorem \ref{thm:real_approximates_pseudo} at the Appendix \ref{prf:real_approximates_pseudo}.
\subsection*{Online convex optimization view}
The adversarial training algorithm fits the framework of \emph{online gradient descent} (OGD): at each step $t$,
\begin{enumerate}
    \item The adversary chooses the loss function $\mathcal{L}_t(W)=\mathcal{L}\left(f_{W^{(t)}}, S^{(t)}\right)$.
    \item The learner incurs the cost $\mathcal{L}_t(W^{(t)})$ and updates $W^{(t+1)} = W^{(t)} - \eta \nabla_W \mathcal{L}_t(W^{(t)})$.
\end{enumerate}
    Online gradient descent comes with regret guarantees, when the loss functions are convex~\citep{h16}, but in our case they are not. However, it can be shown that during adversarial training, the weights stay near initialization, which implies that the net behaves like a pseudo-net. Moreover, pseudo-net is linear in $W$ and so the regret guarantee holds, up to a small approximation error. Notably, the regret is with respect to the best net in hindsight, that is also close to initialization. 
\\


\begin{theorem}\label{thm:convergence}
For all $\epsilon\in(0,1)$, $R\geq 1$, there exists an $M=\poly\left(n, R, \frac{1}{\epsilon}\right)$, such that for every $m\geq M$, with probability at least $1- \exp\left( -\Omega\left(m^{1/3}\right) \right)$ over the choice of $a^{(0)},W^{(0)},b^{(0)}$, if we run Algorithm~\ref{alg:adv-train} with hyper-parameters
\begin{align*}
T=\Theta( \epsilon^{-2}  R^2 )\ \text{and} \ \eta=\Theta(\epsilon   m^{-1/3} )
\end{align*}
then for every $W^*$ such that $\|W^* - W^{(0)}\|_{2,\infty}\leq \frac{R}{m^{2/3}}$,
the output weights $\left(W^{(t)}\right)_{t=1}^T$ satisfy
\begin{align*}
    \frac{1}{T}\sum_{t=1}^T \mathcal{L}_{\mathcal{A}}\left(f_{W^{(t)}}\right) \leq\mathcal{L}_{\mathcal{A}^*}\left(f_{W^*}\right)+ \epsilon
\end{align*}

\end{theorem}

Note that while in the LHS of the guarantee we have the robust losses w.r.t. $\mathcal{A}$, in the RHS we have the worst-case robust loss.
We give the proof of Theorem \ref{thm:convergence} at the Appendix \ref{prf:converge}. 

The connection with OCO was first made in \cite{gclhwl19}. However, they prove the above result for the case of quadratic ReLU activation. For the classical ReLU, they need to enforce the closeness to the initialization during training via a projection step, that is not used in practice. 

\subsection*{Existence of robust network near initialization}
What is left to do to prove Theorem \ref{thm:main} is to show the existence of a network $f_{W^*}$ that is close to initialization and the worst-case robust loss $\mathcal{L}_{\mathcal{A}^*}(f_{W^*})$ is small. \cite{gclhwl19} required $m$ to be at least $\left(\frac{1}{\epsilon}\right)^{\Omega(d)}$ to prove this statement. Our main result is the proof of existence of such network with width at most $\poly\left(d,\left(\frac{n}{\epsilon}\right)^{1/\gamma}\right)$. Formally, for a $\rho$-bounded adversary and  $\gamma$-separable training set, we have the following theorem.



\begin{theorem}\label{thm:existence}
 For all $\epsilon\in(0,1)$, there exists
 \begin{align*}
 M_0=\poly\left(d,\left(\frac{n}{\epsilon}\right)^{1/\gamma}\right)   \   \text{and} \  R=\poly\left(\left(\frac{n}{\epsilon}\right)^{1/\gamma}\right)
 \end{align*}
 such that for every $m\geq M_0$, with probability at least $1- \exp\left( -\Omega\left(m^{1/3}\right) \right)$ over the choice of $a^{(0)},W^{(0)},b^{(0)}$, there exists $W^*\in\mathbb{R}^{d\times m}$ such that $\|W^*-W^{(0)}\|_{2,\infty}\leq \frac{R}{m^{2/3}}$ and
     \begin{align*}
     \mathcal{L}_{\mathcal{A}^*}\left(f_{W^*}\right)\leq \epsilon
     \end{align*}
\end{theorem}

The proof of Theorem \ref{thm:existence} has three steps:
\begin{itemize}
    \item 
    We show that there is a function $f^*:\mathcal{X}\rightarrow \mathbb{R}$ that has "low complexity" and for all datapoints $(x_i,y_i)$ and perturbed inputs $\tilde{x}_i\in \mathcal{B}_2(x_i,r)$, $f^*(\tilde{x}_i) \approx y_i$. More specifically, this function will have the form
    \begin{align*}
        f^*(x)=\sum_{i=1}^n y_i q(\langle x_i, x\rangle )
    \end{align*}
    where $q$ is a low-degree polynomial approximating a step function that is $1$ for $\langle x_i, x\rangle \approx 1$ and $0$ otherwise. The existence of such a low-degree polynomial is proven using tools from approximation theory that appear in \cite{sachdeva2014faster,frostig2016principal}.
    \item 
    We show that since $f^*$ has "low complexity", there exists a pseudo-network $g_{W^*}$ that is close to  initialization, has polynomial width (for $\gamma=\Omega(1)$), and $g_{W^*} \approx f^*$.
    \item We use Theorem \ref{thm:real_approximates_pseudo} to show that for the real network $f_{W^*}$ we have $f_{W^*}\approx g_{W^*}$.
\end{itemize}

We provide a sketch of the implementation of these three steps in section \ref{sec:proofs}.

\section{Proof of Theorem \ref{thm:existence}}\label{sec:proofs}

We first provide the definition of a complexity measure for polynomials, following \cite{all19}. Note that the definitions of that paper also have an input parameter $R$. In this work, we set that $R$ to be 1.
\begin{definition} \label{def:complexities}
Let $c>1$ denote a sufficiently large constant. 
For any degree-$k$ univariate polynomial $\phi (z)= \sum_{j=0}^k \alpha_j z^j $, and parameter $\epsilon_1>0$, we define the following two measures of complexity
\begin{align*}
    \mathfrak{C}(\phi,\epsilon_1) := \sum_{j=0}^k  c^j \cdot ( 1 + ( \sqrt{ \ln(1/\epsilon_1) / j }  )^j ) \cdot |\alpha_j|
\end{align*}

\begin{align*}
    \mathfrak{C}(\phi) := c \cdot \sum_{j=0}^{k} (j+1)^{1.75} |\alpha_j|
\end{align*}
\end{definition}

\subsection{Robust fitting with polynomials}\label{sec:robust_fitting}
In this section we show that the fact that the points $x_i$ in the training set have pairwise $\ell_2$ distance at least $\delta$ and $1/\delta$ is not too large implies that there is a function $f^*$ that has "low complexity" and robustly fits the training set:
\begin{align*}
& \forall i \in [n], \tilde{x}_i\in \mathcal{B}_2(x_i,\rho),\ \ \  f^*(\tilde{x}_i) \approx y_i
\end{align*}

Formally, we prove the following lemma.

\begin{lemma}\label{thm:robust_fitting}
Let $D=\frac{24}{\gamma}\ln\left(48\frac{n}{\epsilon}\right)$. There exists a polynomial $q : \R \rightarrow \R$ with degree at most $D$, size of coefficients at most $O(\gamma^{-1}2^{6D})$, such that for all $j\in [n]$ and $\tilde{x}_j\in \mathcal{B}_2(x_j,\rho)$,
\begin{align*}
\left|\sum_{i=1}^n y_i \cdot q( \langle x_i , \tilde{x}_j \rangle )-y_i \right| \leq \frac{\epsilon}{3}.
\end{align*}
\end{lemma}
Given the polynomial $q$ of the lemma, we will write $f^*(x):=\sum_{i=1}^n y_i \cdot q( \langle x_i , x \rangle )$.
To prove Lemma \ref{thm:robust_fitting}, we first show how to approximate the step function via a polynomial. 
More specifically, the plan is this polynomial to
take as input the inner product of two unit vectors $u,v$ and its output to be close to
\begin{align*}
\begin{cases}
1, & \text{~if~} \|u-v\|_2 \leq \rho ;\\
0, & \text{~if~} \|u-v\|_2 \geq \delta-\rho.
\end{cases}
\end{align*}

Note that since these are unit vectors, $\|u-v\|_2\leq \rho$ is equivalent to $\langle u,v \rangle \geq 1-\rho^2/2$, and $\|u-v\|_2 \geq \delta-\rho$ is equivalent to $ \langle u ,  v \rangle \leq 1-(\delta-\rho)^2/2 $. We prove the following claim.
\begin{claim}\label{lem:step_poly}
Let $\epsilon_1 \in (0,1)$ and $D=\frac{24}{\gamma}\ln\left(\frac{16}{\epsilon_1}\right)$. Then, there exists a univariate polynomial $q_{\epsilon_1}(z)$ with degree at most $D$ and size of coefficients at most $O(\gamma^{-1}2^{6D})$, such that 
\begin{enumerate}
    \item $\forall z\in [1-\rho^2/2,1]$,\ $|q_{\epsilon_1}(z)-1|\leq \epsilon_1$.
    \item $\forall z\in [-1,1-(\delta-\rho)^2/2)$,\ $|q_{\epsilon_1}(z)|\leq \epsilon_1$. 
\end{enumerate}
\end{claim}
\begin{proof}

For $\alpha\in[-1,1]$, we define
\begin{align*}
\step_{\alpha}(z) =
\begin{cases}
		0,  & \text{~if~} -1\leq z < \alpha \\
		1/2,  & \text{~if~} z = \alpha \\
		1, & \text{~if~} \alpha<z \leq 1
\end{cases}
\end{align*}

\begin{align*}
\sgn(z) =
\begin{cases}
		-1,& \text{~if~} -1\leq z < 0 \\
		0, & \text{~if~} z = 0 \\
		1, & \text{~if~} 0<z \leq 1
\end{cases}
\end{align*}
Note that $\step_{\alpha}(z)=\frac{1}{2}(\sgn(z-\alpha)+1)$. We need a polynomial approximation result of the $\sgn$ function, from~\cite{frostig2016principal}.
\begin{lemma}[Lemma 5.5 from~\cite{frostig2016principal}]\label{lem:lemma_5.5_in_fmms16}
Let $\epsilon_1,\eta\in (0,1)$ and $D=\frac{3}{\eta}\ln\frac{2}{\eta\epsilon_1}$. Then, there exists a univariate polynomial $p_{\epsilon_1}(z)=\sum_{j=0}^k \alpha_j z^j$ with degree $k\leq D$ and $|\alpha_j|\leq 2^{4D}$, that is an $\epsilon_1$-approximation of the $\sgn$ function in $[-1,1]\setminus (-\eta,\eta)$, meaning that

\begin{enumerate}
    \item $\forall z\in [\eta,1]$,\ $|p_{\epsilon_1}(z)-1|\leq \epsilon_1$.
    \item $\forall z\in [-1,-\eta]$,\ $|p_{\epsilon_1}(z)+1|\leq \epsilon_1$. 
\end{enumerate}
\end{lemma}

\cite{frostig2016principal} describe how to construct the above polynomial and bound its degree, but do not present a bound on its coefficients. We prove Lemma~\ref{lem:lemma_5.5_in_fmms16} in Appendix \ref{prf:lemma_5.5_in_fmms1}.

We can now approximate the step function by the polynomial 
\begin{align}\label{eq:approx_step_function}
q_{\epsilon_1}(z)=\frac{p_{\epsilon_1}(2(z-\alpha))+1}{2}.
\end{align}
Because of the lemma and the connection between the $\sgn$ and the $\step_\alpha$ functions, we get that $\forall z\in[-2+\alpha,2+\alpha]\setminus[\alpha-2\eta,\alpha+2\eta]$,
\begin{align*}
| q_{\epsilon_1}(z) - \step_\alpha(z) |\leq \epsilon_1/2.
\end{align*}
Observe that $q_{\epsilon_1}$ also has degree $k$ and if $A=\max_j\{|\alpha_j|\}$, then the coefficient of $z^j$ in $q_{\epsilon_1}$ has size at most $2^{k-1}A\sum_{i=j}^k {i \choose j}|\alpha|^{i-j}+1/2\leq \frac{2^{2k-1}A}{1-\alpha}+1/2$
\\
$\leq \frac{2^{6D-1}}{1-\alpha}+1/2 $ . Setting
\begin{align*}
\eta= \delta(\delta-2\rho) / 8\leq \gamma/8 \text{~~~and~~~} \alpha=1-\frac{\rho^2}{2}-2\eta
\end{align*}
 finishes the proof.
\end{proof}

To finish the proof of Lemma \ref{thm:robust_fitting}, let $q$ be the polynomial that we get from Claim \ref{lem:step_poly}, by setting $\epsilon_1=\epsilon/(3n)$. Let $f^*(x)=\sum_{i=1}^n y_i q( \langle x_i, x \rangle )$.
 For all $i,j\in[n]$, $i\neq j$ and $\tilde{x}_i\in {\cal B}_2(x_i,\rho)$,  we have $\|x_j-\tilde{x}_i\|_2\geq \delta-\rho$. Thus, from Claim \ref{lem:step_poly} we have $|q( \langle x_j, \tilde{x}_i \rangle )|,|q( \langle x_i, \tilde{x}_i \rangle )-1|\leq \epsilon/(3n)$. 
\begin{align*}
 |f^*(\tilde{x}_i)-y_i|  &\leq |y_i||1-q(\langle x_i,\tilde{x}_i \rangle)| +\sum_{j\neq i} |y_j|| q( \langle x_j, \tilde{x}_i \rangle )| \\
 &\leq \epsilon/(3n)+(n-1)\epsilon/(3n)\\
 &\leq\epsilon/3 \\
\end{align*}


\subsection{Pseudo-Network Approximates $f^*$}\label{sec:pseudo_approximates_polynomial}

We prove that we can use a pseudo-network with width $\poly\left(d,\left(\frac{n}{\epsilon}\right)^{1/\gamma}\right)$ to approximate $f^*$, uniformly over $\cal{X}$.
\begin{lemma}\label{thm:pseudo_approximates_polynomial}
For all $\epsilon\in(0,1)$, there exist
 \begin{align*}
 M=\poly\left(d,\left(\frac{n}{\epsilon}\right)^{1/\gamma}\right)   \   \text{and} \  R=\poly\left(\left(\frac{n}{\epsilon}\right)^{1/\gamma}\right)
 \end{align*}
 such that for $m\geq M$, with probability at least $1- \exp\left( -\Omega\left(\sqrt{m/n}\right) \right)$ over the choice of $a^{(0)},W^{(0)},b^{(0)}$, there exists there exists a $W^*\in\mathbb{R}^{d\times m}$ such that $\|W^*-W^{(0)}\|_{2,\infty}\leq \frac{R}{m^{2/3}}$ and 
     \begin{align*}
     \sup_{x\in\cal{X}} |g_{W^*}(x) -f^*(x) | \leq \epsilon/3
     \end{align*}
\end{lemma}

 \cite{all19} prove a similar but weaker guarantee, by approximating $f^*$ using a pseudo-network, \emph{ in expectation}. In other words, they show that for some data distribution $\mathcal{D}$, $\mathbb{E}_{x\sim \mathcal{D}}\left[\left|g_{W^*}(x)-f^*(x)\right|\right]$ is small, for some pseudo-network $g_{W^*}$ close to  initialization. As we mentioned previously, dealing with the average case is not enough and we need a uniform approximation guarantee, since we account for adversarial perturbations of the inputs.

We give here a proof sketch for Lemma \ref{thm:pseudo_approximates_polynomial} and the full proof at the Appendix \ref{prf:pseudo_approximates_polynomial}.
We use a technical result from \cite{all19}. Suppose that for a given unit vector $w^* \in \mathbb{R}^d$ and a univariate polynomial $\phi$, we want to approximate the function of a unit vector $x$ given by $\phi(\langle w^* ,x \rangle)$, via a linear combination of random ReLU features. Intuitively, their result says that if $\phi$ has low complexity, then the weights of this linear combination can be small. 
\begin{lemma}[Lemma 6.2 from~\cite{all19}]\label{lem:indicator_to_function}
For every univariate polynomial $\phi : \R \rightarrow \R$, for every $\epsilon_2 \in (0,1/\mathfrak{C}(\phi))$, there exists a function $h:\R^2\rightarrow[- \mathfrak{C}(\phi,\epsilon_2), \mathfrak{C}(\phi,\epsilon_2)]$ such that for all $w^*, x\in\R^d$ with $\|w^*\|_2=\|x\|_2= 1$, we have
 
\begin{align*}
   & \Big|\E_{u \sim {\cal N}(0,I_d), \beta \sim {\cal N}(0,1) } \left[\ind\{ \langle u,x \rangle +\beta\geq0 \} \  h( \langle w^* , u \rangle, \beta)\right] -\phi( \langle w^*,x \rangle ) \Big|\leq \epsilon_2
\end{align*}

\end{lemma}

The above lemma implies $f^*$ can be approximated by an "infinite" pseudo-network. We use concentration bounds to argue that there exists a pseudo-network $g_{W^*}$ with width $\poly\left(d,\left(\frac{n}{\epsilon}\right)^{1/\gamma}\right)$, such that for any fixed input $x\in \mathcal{X}$, with probability at least $1-\exp(-\Omega(\sqrt{m/n}))$, $g_{W^*}(x)\approx f^*(x)$. We conclude the argument via a union bound over a $\frac{1}{\poly(m)}$-net of $\cal{X}$ and a perturbation analysis for $g$, similarly to the proof of Theorem \ref{thm:real_approximates_pseudo}.


\subsection{Putting it all together}
We will use Lemmas \ref{thm:robust_fitting}, \ref{thm:pseudo_approximates_polynomial} and Theorem \ref{thm:real_approximates_pseudo} to prove Theorem \ref{thm:existence}. From Lemma \ref{thm:robust_fitting} we get $f^*$. From Lemma \ref{thm:pseudo_approximates_polynomial} we get the $M$, $R$ and $W^*$ and combining with Theorem \ref{thm:real_approximates_pseudo}, we have that as long as $m\geq \max\{\poly(d),M\}$, with probability at least $p:=1-\exp(-\Omega(\sqrt{m/n}))-\exp(-\Omega(-m^{1/3}))$, there exists a $W^*\in\mathbb{R}^{d\times m}$ such that $\|W^*-W^{(0)}\|_{2,\infty}\leq \frac{R}{m^{2/3}}$ and for all $x\in \mathcal{X}, |g_{W^*}(x)-f^*(x)|\leq \epsilon/3$ and $|f_{W^{*}}(x)-g_{W^{*}}(x)|\leq O\left( \frac{R^2}{m^{1/6}}\right)$. Thus, for all $i\in[n], \tilde{x}_i\in \mathcal{B}(x_i,\rho)$, 
\begin{align*}
   \ell(f_{W^*}(\tilde{x}_i),y_i) & \leq |f_{W^*}(\tilde{x}_i)-y_i| \\
   & \leq |f^*(\tilde{x}_i)-y_i|+|g_{W^{*}}(\tilde{x}_i)-f^*(\tilde{x}_i)| + |f_{W^{*}}(\tilde{x}_i)-g_{W^{*}}
   (\tilde{x}_i)| \\
   &\leq \frac{2\epsilon}{3}+O\left(\frac{R^2}{m^{1/6}}\right) \\
   &\leq \epsilon \\
\end{align*}
since $m\geq \poly\left(d,\left(\frac{n}{\epsilon}\right)^{1/\gamma}\right)$, for a large enough polynomial.
Thus, we have that $L_{\mathcal{A}^*}(f^*)\leq \epsilon$. As for the bound on the probability of success $p$, since $m\geq n^{3}$ (for large enough polynomial in the lower bound for $m$), we get $p\geq 1-\exp(-\Omega(-m^{1/3}))$.





\section{Conclusion and discussion}
We have shown that under a natural separability assumption on the training data, adversarial training on polynomially wide two-layer ReLU networks always converges in polynomial time to small robust training loss, significantly improving previous results. This may serve as an explanation for small loss achieved by adversarial training in practice. Central in our proof is an explicit construction of a robust net near initialization, utilizing ideas from polynomial approximation.

As a future direction, it would be nice to improve the current exponential in $1/\gamma$ width requirement to polynomial. Ideally, the width requirement would fall back to $\poly(1/\gamma)$ as in standard (non-adversarial) training setting when the perturbation radius $\rho$ approaches zero, which is missing in our construction. We believe it may require a better understanding of the expressivity of over-parameterized nets. Furthermore, a natural next step is to extend our results to multi-layer ReLU networks. 
\label{sec:conclusion}
\bibliographystyle{plainnat}
\bibliography{main}

\onecolumn
\appendix
\section{Appendix}

For functions $f,g:\mathbb{\mathcal{X}}\rightarrow \mathbb{R}$, we define 

\begin{align}\label{def:infty_norm}
    \|f-g\|_{\infty}=\sup_{x\in \mathcal{X}}|f(x)-g(x)|
\end{align}

\subsection{Proof of Theorem \ref{thm:real_approximates_pseudo}}\label{prf:real_approximates_pseudo}
\begin{proof}
Let $W\in \mathbb{R}^{d\times m}$, $\|W-W^{(0)}\|_{2,\infty}\leq \frac{R}{m^{2/3}}$, that is arbitrarily correlated with the initialization $a^{(0)}, W^{(0)},b^{(0)}$. It suffices to bound $\|f_W-g_W\|_\infty$, where $\|\cdot\|_{\infty}$ is defined in \ref{def:infty_norm}.
From now on we work with this $W$ and we write $f,g$ for $f_W$, $g_W$. Also, let $\Delta W_r=W_r-W_r^{(0)}$,  \ $\mathbb{I}_{x,r}^{(0)}= \ind\{\langle W_r^{(0)} , x \rangle+b_r^{(0)}\geq 0 \}$ and $\mathbb{I}_{x,r}= \ind\{\langle W_r^{(0)}+\Delta W_r , x \rangle+b_r^{(0)}\geq 0 \}$. So, we write 
\begin{align*}
& f(x)=\sum_{r=1}^m a_r^{(0)} \left(\langle W_r^{(0)}+\Delta W_r , x \rangle+b_r^{(0)}\right)\mathbb{I}_{x,r} \\
& g(x)=\sum_{r=1}^m a_r^{(0)} \langle \Delta W_r , x \rangle \mathbb{I}_{x,r}^{(0)}
\end{align*}

We prove an elementary anti-concentration property of the Gaussian distribution.

\begin{claim}\label{cl:anti-concentration}
    Let $u\sim \mathcal{N}(0, I_d)$ and $\beta \sim \mathcal{N}(0,1)$, which are independent. For all $x\in \mathcal{X}$ and $t\geq0$,
    \begin{align*}
    \Pr[|\langle u, x \rangle + \beta | \leq t  ] = O( t ) .
    \end{align*}
\end{claim}

\begin{proof}
We fix $x$ and $t$ and we have that $\langle u, x \rangle + \beta  \sim \N(0, 2  )$. Moreover,
  \begin{align*}
         \Pr_{z \sim \N(0, 2 )} \Big[ | z | \leq t \Big] 
    ~= \int_{-t}^{t}\frac{1}{\sqrt{2\pi}}e^{ -z^2/4} \d z \leq  ~ \sqrt{\frac{2}{\pi}}t
    \end{align*}
  
\end{proof}
 For all $x\in \mathcal{X}$, $r\in[m]$ and $t\in\mathbb{R}_+$, we define 
\begin{align*}
    \Lambda_r(x,t):=\ind \left\{|\langle W_r^{(0)},x  \rangle + b_r^{(0)}  |\leq t \right\}.
\end{align*}
and observe that from Claim \ref{cl:anti-concentration}, after scaling by $\sqrt{m}$, we have that $\Pr[ \Lambda_r( x , t ) = 1 ] \leq O( t \sqrt{m} )$.


We will prove that for every fixed $x\in \mathcal{X}$, with high probability, $|f(x)-g(x)|$ is small.

\begin{lemma}
For all $x\in \mathcal{X}$, with probability at least $1-\exp(-\Omega(m^{1/3}))$, 
\begin{align}\label{lem:real_net_approx_pseudo_fixed_x}
   | f(x) - g(x) | \leq O( R^2 / m^{1/6} )
    \end{align}
\end{lemma}
\begin{proof}

     Let $A_r:=\ind\{\mathbb{I}_{x,r}\neq \mathbb{I}_{x,r}^{(0)}\}$. We bound the size of $\sum_{r=1}^m A_r$ with the following claim.

     \begin{claim}\label{claim:A_r}
         For all $x\in \mathcal{X}$ ,with probability at least $1-\exp\left(-\Omega\left(m^{5/6}\right)\right)$,
        \begin{align*} 
            \sum_{r=1}^m A_r \leq O( R \cdot m^{5/6} ) 
        \end{align*}
    \end{claim}
    \begin{proof}
We fix an $x\in \mathcal{X}$. Since $\|x\|_2=1$ and $\|\Delta W\|_{2,\infty}\leq R / m^{2/3} $, we have that 
    \begin{align*}
    A_r \leq \ind \left\{|\langle W_r^{(0)} , x \rangle+b_r^{(0)}| \leq \|\Delta W_r\|_2 \right\} \leq \Lambda_r (x, R/ m^{2/3} ).
    \end{align*}
    But, as we mentioned previously, gaussian anti-concentration implies that 
    \begin{align*}
        \Pr\left[\Lambda_r(x, R / m^{2/3} )=1 \right] \leq O( R / m^{1/6} )
    \end{align*} Since for our fixed $x$, these are $[m]$ independent Bernoulli random variables, standard concentration implies that with probability at least $1-\exp(-\Omega(m^{5/6}))$,
    \begin{align*}
    \sum_{r=1}^m\Lambda_r (x, R /  m^{2/3} )\leq O(Rm^{5/6} ).
    \end{align*}
    The fact that $\sum_{r=1}^m A_r\leq \sum_{r=1}^m\Lambda_r(x, R /  m^{2/3} )$ finishes the proof of the claim.
    \end{proof}
    We decompose $f$, using the following three functions
    \begin{definition}
    We define $f_1, f_2, f_3$ as follows:
    \begin{align*}
    f_1(x) := & ~ \sum_{r=1}^m a_r^{(0)} \langle \Delta W_r, x\rangle \mathbb{I}_{x,r} \\
    f_2(x) := & ~ \sum_{r=1}^m a_r^{(0)} (\langle  W_r^{(0)}, x\rangle +b_r^{(0)}) \mathbb{I}_{x,r}^{(0)} \\
    f_3(x) := & ~ \sum_{r=1}^m a_r^{(0)}(\langle  W_r^{(0)}, x\rangle +b_r^{(0)})(  \mathbb{I}_{x,r}-\ \mathbb{I}_{x,r}^{(0)} )
    \end{align*}
    \end{definition}
 
    It is easy to see that $f(x) =  f_1(x) + f_2(x) + f_3(x)$. We proceed by showing that $|f_1(x) - g(x)|,|f_2(x)|$ and $|f_3(x)|$ are all small.
    
    
    
    \begin{claim}\label{cl:bound_f_1}
    With probability at least $1-\exp(-\Omega(m^{5/6}))$, 
     \begin{align*}
    |f_1(x)-g(x)| \leq O( R^2 / m^{1/6} ) 
    \end{align*}
    \end{claim}
    \begin{proof}
     From the definition of $A_r$ we have that $|\mathbb{I}_{x,r}-\mathbb{I}_{x,r}^{(0)}|\leq A_r$. 

    \begin{align*}
         |f_1(x)-g(x)|
         = & ~ \left|\sum_{r=1}^m a_r \langle \Delta W_r, x\rangle (\mathbb{I}_{x,r}-\mathbb{I}_{x,r}^{(0)})\right| \\
         \leq & ~ \sum_{r=1}^m |a_r| \cdot 
        |\langle \Delta W_r, x\rangle | \cdot A_r \\
        \leq & ~ \frac{R}{m}\sum_{r=1}^m A_r 
    \end{align*}
    The last step follows from $\| \Delta W \|_{2,\infty}\leq \frac{R}{m^{2/3}}$, $a_r\sim \{ \pm \frac{1}{m^{1/3}}\}$. From Claim \ref{claim:A_r}, with probability at least $1-\exp(-\Omega(m^{5/6}))$, \begin{align*}
    |f_1(x)-g(x)|\leq O( R^2 / m^{1/6} )
    \end{align*}

    \end{proof}
    
    \begin{claim}\label{cl:bound_f_2}
    With probability at least $1-\exp(-\Omega(m^{1/3}))$,
    \begin{align*}
    |f_2(x)|\leq O( 1 / m^{1/6} )
    \end{align*}
    \end{claim}
    \begin{proof}
    From the definition of $f_2$,
    \begin{align*}
    f_2(x)=\sum_{r=1}^ma_r^{(0)} \sigma(\langle W_r^{(0)} , x\rangle +b_r^{(0)}).
    \end{align*}
    By definition of the ReLU function $\sigma(\cdot)$, 
    \begin{align*}
    \sum_{r=1}^m \sigma^2(\langle W_r^{(0)} , x\rangle +b_r^{(0)}) \leq \sum_{r=1}^m (\langle W_r^{(0)} , x\rangle +b_r^{(0)})^2
    \end{align*}
    Also, note that for $r\in[m]$, $\langle W_r^{(0)} , x\rangle +b_r^{(0)}\sim \N(0,2/m)$ and independent. From concentration of the sum of independent Chi-Square random variables, we have that with probability at least $1-\exp(-\Omega(m))$, 
     \begin{align}
    \sum_{r=1}^m \sigma^2(\langle W_r^{(0)} , x\rangle +b_r^{(0)})\leq & ~ \sum_{r=1}^m (\langle W_r^{(0)} , x\rangle +b_r^{(0)})^2 \\
    = & ~ O(1) 
    \end{align}
    Now, because of independence, using Hoeffding's concentration inequality, for some large constant $c>0$, 
    \begin{align*}
    &\Pr \left[\left|\sum_{r=1}^ma_r^{(0)} \sigma(\langle W_r^{(0)} , x\rangle +b_r^{(0)}) \right|\geq \frac{c }{m^{1/6}}\ \ \Bigg{|} \ W^{(0)},b^{(0)} \right]\\
    & \leq \exp \left(-\Omega\left(\frac{m^{-1/3}}{\frac{1}{m^{2/3}}\sum_{r=1}^m\sigma^2 \left( \langle W_r^{(0)} , x\rangle +b_r^{(0)} \right) }\right) \right)
    \end{align*}
    and using the previous bound we get that overall, with probability at least $1-\exp(-\Omega(m^{1/3}))$, 
    \begin{align*}
    |f_2(x)|\leq O( 1 / m^{1/6} )
    \end{align*}
    \end{proof}

    \begin{claim}\label{cla:bound_f_3}
    With probability at least $1 - \exp( -\Omega( m^{ 5/6 } ) )$, 
    \begin{align*} 
    |f_3(x)| \leq O ( R^2 / m^{ 1/6 } )
    \end{align*}
    \end{claim}
    \begin{proof}
    \begin{align*}
    |f_3(x)|= & ~ \Big| \sum_{r=1}^m a_r^{(0)}(\langle  W_r^{(0)}, x\rangle +b_r^{(0)})(  \mathbb{I}_{x,r}-\ \mathbb{I}_{x,r}^{(0)} ) \Big|\\
    \leq & ~ \sum_{r=1}^m |a_r^{(0)}| \left|\langle  W_r^{(0)}, x\rangle +b_r^{(0)} \right| \left|  \mathbb{I}_{x,r}-\ \mathbb{I}_{x,r}^{(0)} \right|\\
    \leq & ~ \frac{1}{m^{1/3}} \sum_{r=1}^m \left|\langle  W_r^{(0)}, x\rangle +b_r^{(0)} \right| \left|  \mathbb{I}_{x,r}-\ \mathbb{I}_{x,r}^{(0)}  \right|
    \end{align*}

    We use that 
    \begin{align*}
    \left|  \mathbb{I}_{x,r}-\ \mathbb{I}_{x,r}^{(0)} \right|\leq A_r\leq \Lambda_r( x , R / m^{2/3} ).
    \end{align*}
    Now, remember that $\Lambda_r( x , R / m^{2/3} )\neq 0\iff \left|\langle  W_r^{(0)}, x\rangle +b_r^{(0)} \right| \leq R / m^{2/3} $, so
    \begin{align*}
    \left|\langle  W_r^{(0)}, x\rangle +b_r^{(0)} \right| \left|  \mathbb{I}_{x,r}-\ \mathbb{I}_{x,r}^{(0)}  \right|
    \leq  ~ \left|\langle  W_r^{(0)}, x\rangle +b_r^{(0)} \right| \Lambda_r(x, R / m^{2/3}) \leq  ~ \frac{R}{m^{2/3}}\Lambda_r(x,R/m^{2/3})
    \end{align*}
    Thus, 
    \begin{align*}
    |f_3(x)|\leq \frac{R}{m}\sum_{r=1}^m\Lambda_r(x, {R} / {m^{2/3}}).
    \end{align*}
    But, as we previously showed, with probability at least $1-\exp(-\Omega(m^{5/6}))$, 
    \begin{align*}
    \sum_{r=1}^m\Lambda_r(x, {R} / {m^{2/3}})\leq O(Rm^{5/6}).
    \end{align*}
    Thus, with probability at least $1-\exp(-\Omega(m^{5/6}))$, 
    \begin{align*}
    |f_3(x)| = O ( R^2 / m^{1/6} ).
    \end{align*}

    \end{proof}

We are ready to finish the proof of the lemma \ref{lem:real_net_approx_pseudo_fixed_x}. Aggregating these three claims with a union bound, we have that for every $x\in \mathcal{X}$, with probability at least 
$
1 - \exp(-\Omega(m^{5/6}))-\exp(-\Omega(m^{1/3})) 
= 1 - \exp(-\Omega(m^{1/3})),
$
we have
\begin{align}\label{eq:guarantee_fixed_x}
   |f(x)-g(x)|\leq ~  |f_1(x)-g(x)|+|f_2(x)|+|f_3(x)| \leq  ~ O ( R^2 / m^{1/6} )
\end{align}
\end{proof}

What is left to do is to "union bound" over all ${\cal X}$. Of course, there is the problem that ${\cal X}$ is uncountable. So, we first do a union bound over a very fine-grained net of ${\cal X}$ and then argue about the change of $f$ and $g$ when we slightly change the input $x$.

Let ${\cal X}_1$ be a maximal $\frac{1}{m}$-net of ${\cal X}$. It is well-known that $|{\cal X}_1|\leq\left(\frac{1}{m}\right)^{O(d)}$. From lemma \ref{lem:real_net_approx_pseudo_fixed_x}, by applying a union bound over ${\cal X}_1$, we have that for $m\geq cd^{3}$, where $c$ is a large constant, with probability at least 
\begin{align*}
1 - \exp(O(d\log m)) \cdot \exp(-\Omega(m^{1/3}))=1-\exp(-\Omega(m^{1/3})),
\end{align*}
we have
\begin{align}\label{eq:net_guarantee}
  \forall x\in {\cal X}_1,~~~ |f(x)-g(x)|\leq O ( R^2 / m^{1/6} )
\end{align}

  The final step is the perturbation analysis. We show the following lemma, that applies for fixed inputs.
  
  \begin{lemma}\label{lem:perturbation_analysis_pseudonet}
    For all $x\in \mathcal{X}_1$, with probability at least $1-\exp( - \Omega ( m^{1/2} ) )$, for all $v\in \mathbb{R}^d$, such that $x+v\in {\cal X}$ and $\|v\|_2\leq \frac{1}{m}$, we have 
  \begin{align}\label{eq:perturb_f}
    | f( x + v ) - f( x ) | \leq O( 1/m^{1/3} + R/m )
  \end{align}
  and 
   \begin{align}\label{eq:perturb_g}
     | g( x + v ) - g( x ) | \leq O( R / m^{1/2} ).
  \end{align}
  \end{lemma}

  With this lemma at hand, we can do a union bound over $\mathcal{X}_1$ and conclude that with probability at least $1-\exp(O(d\log m))\exp\left(-\Omega\left(m^{1/2}\right)\right)=1-\exp\left(-\Omega\left(m^{1/2}\right)\right)$ (since $m\geq c d^3$ and $c$ is a large constant), we have that for all $x\in \mathcal{X}_1$ and $v\in \mathbb{R}^d$, such that $x+v\in {\cal X}$ and $\|v\|_2\leq \frac{1}{m}$, the perturbation guarantees \ref{eq:perturb_f} and \ref{eq:perturb_g} hold. Combining this with \ref{eq:net_guarantee} and applying a union bound, we have that with probability at least $1-\exp(-\Omega(m^{1/3}))-\exp\left(-\Omega\left(m^{1/2}\right)\right)=1-\exp(-\Omega(m^{1/3}))$, 
  \begin{align*}
 \|f-g\|_\infty\leq O (R^2/m^{1/6}+1/m^{1/3}+R/m+R / m^{1/2} )=O(R^2/m^{1/6})
  \end{align*}
  and this concludes the proof of theorem \ref{thm:real_approximates_pseudo}.
  \end{proof}

  It remains to prove the Lemma \ref{lem:perturbation_analysis_pseudonet}.

  Let $v$ be a small perturbation of $x$ with the properties stated in the lemma, that can depend arbitrarily on $a^{(0)},W^{(0)},b^{(0)}$.
  \begin{align*}
    |f(x+v)-f(x)|=
    & ~ \Big|\sum_{r=1}^m a_r^{(0)} \Big(\sigma \left(\langle W_r^{(0)}+\Delta W_r , x+v \rangle+b_r^{(0)}\right) -\sigma \left(\langle W_r^{(0)}+\Delta W_r , x \rangle+b_r^{(0)}\right)\Big)\Big| \\
     \leq & ~ \sum_{r=1}^m |a_r^{(0)}| \left |\langle W_r^{(0)}+\Delta W_r , v \rangle \right |\\
    \leq & ~ \frac{1}{m}\sum_{r=1}^m |a_r^{(0)}| \|W_r^{(0)}+\Delta W_r\|_2 \\
    = & ~ \frac{1}{m^{1+1/3}}\sum_{r=1}^m  \|W_r^{(0)}+\Delta W_r\|_2 \\
    \leq & ~ \frac{1}{m^{4/3}}\sum_{r=1}^m  \|W_r^{(0)}\|_2 + \frac{1}{m^{4/3}}\sum_{r=1}^m \|\Delta W_r\|_2\ \\
    \leq & ~ \frac{1}{m^{4/3}}\sum_{r=1}^m  \|W_r^{(0)}\|_2 + \frac{R}{m} 
   \end{align*}
   
   We show the following claim, which concludes the proof of \ref{eq:perturb_f}.
   
   \begin{claim}\label{cl:bound_W_0}
     With probability at least $1-\exp(-\Omega( m))$, $\|W^{(0)}\|_{2,\infty}\leq O(1)$.
   \end{claim}
   \begin{proof}
   From concentration of sum of independent Chi-Square random variables, we have that for all $r$, with probability at least $1-\exp(-\Omega( m^2/d))$, $\|W_r^{(0)}\|_2^2\leq O(1)$. Since $m\geq d$, a union bound over all $r$ finishes the proof of the claim.
   \end{proof}
   
   We now argue about g.

   \begin{align*}
    &~|g(x+v)-g(x)| \\
    &= ~\Bigg|\sum_{r=1}^m a_r^{(0)} \langle \Delta W_r , x+v  \rangle \ind\{\langle W_r^{(0)} , x+v \rangle+b_r^{(0)}\geq 0 \} -\sum_{r=1}^m a_r^{(0)} \langle \Delta W_r , x  \rangle \ind\{\langle W_r^{(0)} , x \rangle+b_r^{(0)}\geq 0\} \Bigg| \\
    &\leq \frac{1}{m}\sum_{r=1}^m |a_r^{(0)}| \|\Delta W_r\|_2 +\sum_{r=1}^m  |a_r^{(0)}| \ |\langle \Delta W_r ,x \rangle |\  \big|\ind\{\langle W_r^{(0)} , x+v \rangle+b_r^{(0)}\geq 0\} - \ind\{\langle W_r^{(0)} , x \rangle+b_r^{(0)}\geq 0\}\big| \\
    &\leq\frac{R}{m}+\frac{R}{m}\sum_{r=1}^m    \Big|\ind\{\langle W_r^{(0)} , x+v \rangle+b_r^{(0)}\geq 0\} - \ind\{\langle W_r^{(0)} , x \rangle+b_r^{(0)}\geq 0\} \Big| . 
   \end{align*} 
 About the last sum, from Claim \ref{cl:bound_W_0}, 
    $\|W^{(0)}\|_{2,\infty}\leq O(1)$ and in this case,
\begin{align*}
 \sum_{r=1}^m \Big|\ind \left\{\langle W_r^{(0)} , x+v \rangle+b_r^{(0)}\geq 0\} - \ind\{\langle W_r^{(0)} , x \rangle+b_r^{(0)}\geq 0\right\} \Big|  \leq \sum_{r=1}^m \Lambda_r  \left(x,O(1/m)\right)
\end{align*} 
   From Claim \ref{cl:anti-concentration}, we have that  $\Lambda_r(x,O(1/m))=1$ with probability at most $O\left(\frac{1}{m^{1/2}}\right)$. Since $x$ is fixed, these are $m$ independent Bernoulli random variables and from standard concentration, with probability at least $1-\exp(-\Omega(\sqrt{m}))$, 
   \begin{align*}
       \sum_{r=1}^m \Lambda_r( x , O( 1 / m ) ) \leq O(\sqrt{m}).
   \end{align*}
   This finishes the proof of \ref{eq:perturb_g}.

\subsection{Proof of Theorem~\ref{thm:convergence}}
\label{prf:converge}

\begin{proof}
We will give the values of $T$ and $\eta$, later in the proof. For simplicity, we use the following shorthand notations to denote various distances.
\begin{itemize}
    \item[] $D_{\text{max}}:= \max_{ t \in [T] }\|W^{(t)} - W^{(0)}\|_{2,\infty}$
    \item[] $D_{W^{*}}:= \|W^{*} - W^{(0)}\|_{2,\infty}$
\end{itemize}
By condition, we know $D_{W^*}=O\left(\frac{R}{m^{2/3}}\right)$.

Even though in Algorithm~\ref{alg:adv-train} the parameters $W$ are updated using the gradients of the \emph{real net}, in this proof we consider the pseudo-net as the object being optimized. Thus we need to relate the real net gradients to the pseudo-net gradients. For ease of presentation, we define the following convenient notations for the two notions of gradients:
\begin{itemize}
    \item[] real net gradient $\nabla^{(t)}:=\nabla_W\mathcal{L}(f(W^{(t)}), S^{(t)})$ 
    \item[] pseudo-net gradient $\hat{\nabla}^{(t)}:=\nabla_W\mathcal{L}(g(W^{(t)}), S^{(t)})$
\end{itemize}
We write both gradients as matrices in $\mathbb{R}^{d\times m}$
In fact, by Lemma~\ref{lemma:coupling}, we know that they are coupled with high probability, as long as $W^{(t)}$ stays close to initialization (i.e., $D_{\max}\leq m^{-15/24}$).
\begin{align*}
      \| \hat{\nabla}^{(t)} - \nabla^{(t)} \|_{2,1} \leq O\left(nm^{13/24}\right)
\end{align*}

\begin{remark}
We assume for now $D_{\max}\leq m^{-15/24}$ is true and in the end we will set proper values for $T, \eta$ and $m$ to make sure this is indeed the case.
\end{remark}
Using the fact that the loss is $1$-Lipschitz, we bound the gradient size:
\begin{align}\label{ineq:gradient-size}
    \|\nabla^{(t)}_r\|_2 \leq |a_r| \left(\frac{1}{n}\sum_{i=1}^n \sigma'\left(\inner{W^{(t)}_r}{x_i}+b_r^{(0)}\right)\|\tilde{x}_i\|_2\right) \leq \frac{1}{m^{1/3}}
\end{align}

Due to the linearity of $g$ with respect to $W$, the loss $\mathcal{L}(g(W), S)$ is convex in $W$. For two matrices $A,B$ with the same dimensions, we write their inner product as $\inner{A,B}:=\mathrm{tr}(A^TB)$.
\begin{align*}
    & ~ \mathcal{L}(g(W^{(t)}), S^{(t)}) - \mathcal{L}(g(W^*), S^{(t)}) \\
    \leq & ~\inner{\nabla^{(t)}}{W^{(t)} - W^*} + \inner{\hat{\nabla}^{(t)} - \nabla^{(t)}}{W^{(t)} - W^*} \\
    \leq & ~ \underbrace{\inner{\nabla^{(t)}}{W^{(t)} - W^*}}_{:=\text{\textalpha}^{(t)}}  + \underbrace{\|\hat{\nabla}^{(t)} - \nabla^{(t)}\|_{2,1}\|W^{(t)} - W^* \|_{2,\infty}}_{:=\text{\textbeta}^{(t)}}
\end{align*}
We deal with $\textalpha^{(t)}$~and~$\textbeta^{(t)}$~terms separately. As for the former, we use the standard online gradient descent proof technique:
\begin{align*}
\| W^{(t+1)} - W^* \|_F^2=\| W^{(t)}-\eta \nabla^{(t)}  - W^* \|_F^2 = \| W^{(t)} - W^* \|_F^2-2\eta \alpha^{(t)}+\eta^2  \|\nabla^{(t)}\|_{F}^2
\end{align*}
So, by rearranging we get
\begin{align*}
\text{\textalpha}^{(t)} &\leq \frac{\eta}{2} \| \nabla^{(t)} \|_F^2 + \frac{\| W^{(t)} - W^* \|_F^2 - \| W^{(t+1)} - W^* \|_F^2}{2\eta} 
\end{align*}
and then sum over $t$,
\begin{align*}
\sum_{t=1}^T \text{\textalpha}^{(t)} &\leq \frac{\eta}{2}\sum_{t=1}^T\|\nabla^{(t)}\|_F^2 + \frac{\| W^{(0)}-W^*  \|_F^2-\| W^{(T+1)}-W^* \|_F^2}{2\eta} \leq \frac{\eta m^{1/3}}{2}T + \frac{m D_{W^*}^2}{2\eta}
\end{align*}
where we used the fact $\|W^* -  W^{(0)}\|_F^2\leq m\cdot \|W^* -  W^{(0)}\|_{2,\infty}=m D_{W^*}^2$ 
as well as $\|\nabla^{(t)}\|_F^2\leq \sum_{r=1}^m \|\nabla_r^{(t)}\|_2^2\leq m^{1/3}$.

For the $\text{\textbeta}^{(t)}$'s, we first invoke Lemma~\ref{lemma:coupling} and then apply triangle inequality:
\begin{align*}
\text{\textbeta}^{(t)} \leq O\left(nm^{13/24}\right)\| W^{(t)} - W^* \|_{2,\infty} \leq O\left(nm^{13/24}\right) \left(D_{\text{max}} + D_{W^{*}}\right)
\end{align*}

Furthermore we can bound the size of $D_{\text{max}}$ using the bound on gradients, i.e. $\|\nabla^{(t)}_r\|_2\leq m^{-1/3}$ using inequality~\ref{ineq:gradient-size}.
\begin{align*}
    D_{\text{max}} &= \max_{ t \in [T]}\|W^{(0)} - W^{(t)}\|_{2,\infty} \leq \sum_{t=1}^T \eta \max_{r\in[m]}\|\nabla^{(t)}_r\|_2 \leq \frac{\eta  T}{m^{1/3}}
\end{align*}

Putting it together with the condition $D_{W*}=O\left(\frac{R}{m^{2/3}}\right)$ that we already have, we obtain the following:
\begin{align*}
    & ~ \sum_{t=1}^T  \mathcal{L}(g(W^{(t)}), S^{(t)}) - \sum_{t=1}^T \mathcal{L}(g(W^*), S^{(t)})\\ 
    \leq & \sum_{t=1}^T \text{\textalpha}^{(t)} + \sum_{t=1}^T \text{\textbeta}^{(t)} \\
    \leq &~ O(1)
    \left(    m^{1/3} \eta T +  \frac{R^2}{m^{1/3} \eta} +\eta T n m^{5/24} + \frac{\eta R T n}{m^{1/8}}\right)
\end{align*}
We then have
\begin{align*}
   \frac{1}{T} \sum_{t=1}^T \mathcal{L}(g(W^{(t)}), S^{(t)}) - \frac{1}{T} \sum_{t=1}^T \mathcal{L}(g(W^*), S^{(t)}) &\leq O(\varepsilon)
\end{align*}
if we set the hyper-parameters $T, m, \eta$ to be the following:
\begin{itemize}
    \item[] $T=\Theta( \epsilon^{-2} R^2 )$,
    \item[] $m\geq \Omega\left(\max\left\{ n^{8}, \left(\frac{Rn}{\epsilon}\right)^{24/11}, \left(\frac{R^2}{\epsilon}\right)^{24}\right\} \right)$,
    \item[] $\eta=\frac{R}{m^{1/3} \sqrt{T}}=\Theta(m^{-1/3} \epsilon )$  
\end{itemize}
Note the the requirement on $m$ is to satisfy $\eta T n m^{1/4} + \frac{\eta R T n}{m^{1/12}} \leq O(\epsilon)$, $D_{\max}\leq m^{-15/24}$ as well as to meet the condition for invoking Theorem~\ref{thm:real_approximates_pseudo}:
\begin{align*}
    \forall t\in[T], \sup_{x\in\mathcal{X}}\left| f_{W^{(t)}}(x) - g_{W^{(t)}}(x) \right| \leq O(\epsilon)
\end{align*}

Thus, we get
\begin{align*}
   \frac{1}{T} \sum_{t=1}^T \mathcal{L}(f_{W^{(t)}}, S^{(t)}) - \frac{1}{T} \sum_{t=1}^T \mathcal{L}(f_{W^*}, S^{(t)}) &\leq c \cdot\varepsilon
\end{align*}
where $c>0$ is a large constant.
Now, observe that $\mathcal{L}(f_{W^{(t)}}, S^{(t)})=\mathcal{L}_{\mathcal{A}}(f_{W^{(t)}})$ and $\mathcal{L}(f_{W^*}, S^{(t)})\leq \mathcal{L}_{\mathcal{A}^*}(f_{W^*}) $. The proof we presented holds for all $\epsilon>0$, so by using $\frac{\epsilon}{c}$ in place of $\epsilon$, we get the desired result.
\end{proof}

\subsection{Gradient coupling}
\begin{lemma} 
\label{lemma:coupling}
With probability at least $1-\exp(-\Omega(m^{1/3}))$, for all iterations $t$ that $\| W^{(t)} - W^{(0)}\|_{2,\infty}\leq O\left(m^{-15/24}\right)$, we have
\begin{align*}
  \| \hat{\nabla}^{(t)} - \nabla^{(t)} \|_{2,1} \leq O\left(  n m^{13/24} \right)
\end{align*}
\end{lemma}
\begin{proof}
We first prove the following claim.
\begin{claim}\label{claim:aux-couple}
With probability probability at least $1-\exp(-\Omega(m^{1/3}))$ over the initialization, for all subsets $\{x_1,\ldots,x_n\}\subseteq\mathcal{X}$ with $n$ points and any $\|\Delta W_r\|_2\leq m^{-15/24}$,
\begin{align*}
\sum_{r=1}^m \ind\left\{\exists i\in[n],~\sgn\left(\inner{W_r^{(0)} + \Delta W_r}{x_i} + b_r^{(0)}\right) \not= \sgn\left(\inner{W_r^{(0)}}{x_i} + b_r^{(0)}\geq 0 \right) \right\}\leq O\left(n m^{7/8}\right)
\end{align*}
\end{claim}

\begin{proof}
We first prove the above result for a fixed set of $n$ points, and then apply a union bound over all possible such sets.
For a fixed set of $n$ points $\{x_1,\ldots,x_n\}\subseteq\mathcal{X}$, we define 
\begin{align*}
 B_r:=\ind\left\{\exists i\in[n],~\sgn\left(\inner{W_r^{(t)}}{x_i} + b_r^{(0)}\right) \not= \sgn\left(\inner{W_r^{(0)}}{x_i} + b_r^{(0)}\geq 0 \right) \right\}   
\end{align*}
 and the goal is to bound the size of $\sum_{r=1}^m B_r$.

We know by~Claim~\ref{cl:anti-concentration} that for each $x_i$ we have
\begin{align*}
    \Pr\left[|\inner{W_r^{0}}{x_i}+b_r^{(0)}|\leq  m^{-15/24} \right]\leq O\left(m^{-1/8}\right)
\end{align*}
With a union bound over the indices $i\in[n]$, we have
\begin{align*}
    \Pr\left[\exists i\in[n], ~|\inner{W_r^{0}}{x_i}+b_r^{(0)}|\leq m^{-15/24} \right]\leq O\left(n m^{-1/8}\right)
\end{align*}
which implies
\begin{align*}
     \Pr\left[ B_r = 1\right] \leq \Pr\left[\exists i\in[n], ~|\inner{W_r^{0}}{x_i}+b_r^{(0)}|\leq m^{-15/24} \right]\leq O\left(n m^{-1/8}\right)
\end{align*}
Because $x_i$'s are fixed for now, $B_r$'s are $m$ independent Bernoulli random variables. Standard concentration implies that with probability at least $1-\exp(-\Omega(nm^{7/8}))$
\begin{align*}
    \sum_{r=1}^m B_r \leq O\left(n m^{7/8}\right)
\end{align*}

As a last step, we take a union bound over a $\frac{1}{m}$-net over product space $\otimes^{n}\mathcal{X}$ which amplifies the failure probability negligibly by only $\exp(O(nd\log m))$ compared to $\exp(-\Omega(m^{1/3}))$ (for large enough $m$).
\end{proof}

Now, we are ready to finish the proof of the coupling lemma.
Remember that $D_{\max}= \| W^{(t)} - W^{(0)}\|_{2,\infty}$. By Claim~\ref{claim:aux-couple}, with probability at least $1-\exp(-\Omega(m^{1/3}))$, all $t$,
\begin{align*}
    \sum_{r=1}^m \ind\left\{\nabla_r^{(t)}=\hat{\nabla}_r^{(t)}\right\}\leq O\left(nm^{7/8}\right)
\end{align*}

For the indices $r$'s that $\nabla_r^{(t)}\not=\hat{\nabla}_r^{(t)}$, we have
\begin{align*}
\| \hat{\nabla}_{r}^{(t)} - \nabla_{r}^{(t)} \|_2 &\leq  |a_r|\frac{1}{n} \sum_{i=1}^n \left|\ind\{\inner{W_r^{(t)}}{\tilde{x}_i} + b_r^{(0)} \geq 0 \} - \ind\{\inner{W_r^{(0)}}{x_i} + b_r^{(0)}\geq 0 \}\right| \|\tilde{x}_i\|_2 \\
&\leq \frac{1}{m^{1/3}} \frac{1}{n} \sum_{i=1}^n \left|\ind\{\inner{W_r^{(t)}}{\tilde{x}_i} + b_r^{(0)} \geq 0 \} - \ind\{\inner{W_r^{(0)}}{x_i} + b_r^{(0)}\geq 0 \}\right|\\
&\leq \frac{1}{m^{1/3}}
\end{align*}
Thus, we conclude
\begin{align*}
    \| \hat{\nabla}^{(t)} - \nabla^{(t)} \|_{2,1} &= \sum_{r=1}^m \| \hat{\nabla}_{r}^{(t)} - \nabla_{r}^{(t)} \|_2 \leq \frac{1}{m^{1/3}} \cdot O\left(nm^{7/8}\right) = O\left(  n m^{13/24} \right)
\end{align*}
\end{proof}

\subsection{Proof of lemma \ref{lem:lemma_5.5_in_fmms16}}\label{prf:lemma_5.5_in_fmms1}
Let 
\begin{align}
    p_k(z):=z\sum_{i=0}^k(1-z^2)^i\prod_{j=1}^i\frac{2j-1}{2j}
\end{align}
\\

\begin{lemma}[Corollary 5.4 in \cite{frostig2016principal}]\label{lem:frostig's-lemma}
    If $z\in [-1,1]$ with $|z|\geq \eta >0$ and $k=\frac{1}{\eta^2}\ln{(2/\epsilon_1)}$, then $|\sgn(z)-p_k(z)|\leq \epsilon_1/2$. Moreover, $p_k$ has degree $2k+1$.
\end{lemma}

We will now compress $p_k$ using Chebyshev polynomials. Recall that the Chebyshev polynomials of the first kind are defined as $T_0(z)=1,T_1(z)=z$ and 
\begin{align}
    T_{k+1}(z)=zT_{k}(z)-T_{k-1}(z)
\end{align}
The definition is also extended for negative $k$ as $T_{-k}(z)=T_{k}(z)$.

We will use the closed-form formula of $T_k(z)$:
\begin{align}\label{eq:chebyshev-closed}
    T_{k}(z)=\sum_{i=0}^{\lfloor n/2 \rfloor} {n \choose 2i}(z^2-1)^iz^{k-2i}
\end{align}

We bound the magnitude of the coefficients of the Chebyshev polynomials via the following proposition.

\begin{proposition}\label{prop:coeff-chebyshev}
The magnitude of the coefficients of $T_{k}(z)$ is at most $2^{2k}$.
\end{proposition}
\begin{proof}
    From the closed-form formula in \ref{eq:chebyshev-closed}, we have that 
    
    \begin{align*}
        T_{k}(z)=\sum_{i=0}^{\lfloor k/2 \rfloor} {k \choose 2i}\sum_{j=0}^i { i \choose j}z^{2j}(-1)^{i-j}z^{k-2i}=\sum_{i=0}^{\lfloor k/2 \rfloor}\sum_{j=0}^i {k \choose 2i} { i \choose j}(-1)^{i-j}z^{k+2j-2i}
    \end{align*}
    The monomials that appear in the above polynomial are the $z^{k-2u}$, for $u=0,\dots, \lfloor k/2 \rfloor$. The magnitude of the coefficient of $z^{k-2u}$ is at most
    
    \begin{align*}
        \left|\sum_{i=u}^{\lfloor k/2 \rfloor} {k \choose 2i} { i \choose i-u}(-1)^{u}\right|\leq \sum_{i=u}^{\lfloor k/2 \rfloor} {k \choose 2i} { i \choose u}\leq \sum_{i=0}^{k} {k \choose i} { k \choose \lfloor k/2 \rfloor}\leq 2^{2k}
    \end{align*}
\end{proof}

Now, let $s$ be a positive integer, $Y_1,\dots,Y_s$ iid $\pm 1$ random variables and $D_s:=\sum_{i=1}^sY_i$. Also, let $D\geq 0$. We define

\begin{align}
    p_{s,D}(z):=\mathbb{E}_{Y_1,\dots,Y_s}[T_{D_s}(z) \ind\{|D_s|\leq D\}]
\end{align}

A straightforward consequence of the proposition \ref{prop:coeff-chebyshev} is the following corollary.

\begin{corollary}\label{cor:bound-psd}
   $ p_{s,D}(z)$ has degree at most $D$ and its coefficients have magnitude at most $2^{2D}$.
\end{corollary}

We will use the following theorem from~\cite{sachdeva2014faster}.

\begin{theorem}[Theorem 3.3 from~\cite{sachdeva2014faster}]\label{thm:thm3.3_approx_theory}
For all positive integers $s,D$ and for all $z\in[-1,1]$,
\begin{align}
  |p_{s,D}(z)-z^s|\leq 2e^{-D^2/(2s)}  
\end{align}
\end{theorem}

Now, we are ready to compress $p_k$. Let $\tilde{p}_k(z):=\sum_{i=0}^k z \left(\prod_{j=1}^i\frac{2j-1}{2j}\right) p_{i,D}(1-z^2)$. Also, let $D=\sqrt{2k\ln(4k/\epsilon_1)}$.

From the above theorem, we have that for all $z\in[-1,1]$,

\begin{align}
    &|\tilde{p}_k(z)-p_k(z)|= \left| \sum_{i=0}^k z\left(\prod_{j=1}^i\frac{2j-1}{2j}\right)\left(p_{i,D}(1-z^2)-(1-z^2)^i \right)\right|\leq  \sum_{i=0}^k \left|p_{i,D}(1-z^2)-(1-z^2)^i \right| \\
    & \leq  \sum_{i=0}^k 2e^{-D^2/(2i)}  \leq \epsilon_1/2
\end{align}

Combining with lemma \ref{lem:frostig's-lemma}, we get that for $k=\frac{1}{\eta^2}\ln(2/\epsilon_1)$, for all $z\in[-1,1]$, $|\sgn(z)-\tilde{p}_{k}(z)|\leq \epsilon_1$. Let $p_{\epsilon}(z):=\tilde{p}_{k}(z)$. We already know that the degree of $p_{\epsilon}(z)$ is at most $D=\frac{1}{\eta}\sqrt{2\ln(2/\epsilon_1)\ln\left(4 \frac{\ln(2/\epsilon_1)}{\eta^2\epsilon_1}\right)}\leq \frac{3}{\eta}\ln(2/(\eta \epsilon_1))$.

It remains to bound the magnitude of its coefficients. Let $p_{i,D}(z)=\sum_{j=0}^D\alpha_jz^j$. From Corollary~\ref{cor:bound-psd}, we have that $\alpha_{max}:=\max_{j}|\alpha_j|\leq 2^{2D}$. Now,
\begin{align*}
p_{i,D}(1-z^2)=\sum_{j=0}^D\alpha_j(1-z^2)^j=\sum_{j=0}^D\alpha_j\sum_{u=0}^j {j \choose u}(-1)^{u} z^{2u}
\end{align*}
The magnitude of the coefficient of $z^{2u}$ is at most $(D+1)\cdot\alpha_{max} \cdot {D \choose u}\leq (D+1) 2^{3D}\leq 2^{4D}$, since $D\geq 1.4$.

\subsection{Proof of Lemma \ref{thm:pseudo_approximates_polynomial}}\label{prf:pseudo_approximates_polynomial}

We will first prove that we can approximate the individual components of $f^*$ via pseudo-networks and then we aggregate these  to form a large pseudo-network that approximates $f^*$.

 \begin{lemma}\label{lem:approx_indiv_components}
    Let $i\in [n]$, $q:\mathbb{R}\rightarrow \mathbb{R}$ univariate polynomial and $\epsilon_3\in\left(0, \frac{1}{\mathfrak{C}(q)}\right)$. Let $\tilde{m}\geq c_1  \frac{d}{\epsilon_3^2} \mathfrak{C}^2(q,\epsilon_3)$, for a large constant $c_1$. For all $r\in[\tilde{m}]$, $U_r^{(0)}\sim \mathcal{N}(0, I_d)$, $\beta_r^{(0)}\sim \mathcal{N}(0, 1)$, $\alpha_r^{(0)}\sim unif\{\pm\frac{1}{m^{1/3}} \}$ and all these random variables and vectors are independent. With probability at least $1-\exp\left(-\Omega\left(\sqrt{\tilde{m}}\right)\right)$, there exists a matrix $\Delta W^{(i)} \in \mathbb{R}^{d\times \tilde{m}}$ with $\|\Delta W^{(i)}\|_{2,\infty}\leq O\left(m^{1/3}\frac{\mathfrak{C}\left(q,\epsilon_3\right)}{\tilde{m}}\right)$ such that 
\begin{align*}
   & \forall x\in \mathcal{X}, \\
   &\left| \sum_{r=1}^{\tilde{m}} \alpha_r^{(0)}\langle \Delta W_r^{(i)},x \rangle \ind\{\langle U_r^{(0)},x \rangle +\beta_r^{(0)} \geq 0\}-y_iq(\langle x_i,x \rangle)\right| \leq 3\epsilon_3
\end{align*}
 \end{lemma}

With this Lemma at hand, we can finish the proof of Lemma \ref{thm:pseudo_approximates_polynomial}.  We apply it for all $i\in [n]$, with $q(z)$ being the polynomial that is given to us by Lemma \ref{thm:robust_fitting}. We now that the degree of $q$ is at most $D$ and the size of its coefficients is at most $c_2\frac{1}{\gamma}2^{6D}$  where $D=\frac{24}{\gamma}\ln(48n/\epsilon)$ and $c_2>0$ is a constant. Using this information about $q$, we can bound its complexities $\mathfrak{C}(q)$ and $\mathfrak{C}\left(q,\epsilon_3\right)$, defined in \ref{def:complexities}, where $\epsilon_3$ will be set after we bound $\mathfrak{C}(q)$ (since from Lemma \ref{lem:approx_indiv_components} $\epsilon_3<1/\mathfrak{C}(q)$). About $\mathfrak{C}(q)$, we directly have $\mathfrak{C}(q)\leq c \cdot c_2\sum_{j=0}^D(j+1)^{1.75}\frac{1}{\gamma}2^{6D}<c \cdot c_2 \frac{(D+1)^{2.75}}{\gamma}2^{6D}$. We set $\epsilon_3=\left(c \cdot c_2 \frac{(D+1)^{2.75}}{\gamma}2^{6D}\right)^{-1}$. About $\mathfrak{C}\left(q,\epsilon_3\right)$, we have

\begin{align}
    &\mathfrak{C}\left(q,\epsilon_3\right)\leq c_2\sum_{j=0}^{D}  c^j  \left( 1 + \sqrt{\ln(1/\epsilon_3) / j }   \right)^j \frac{1}{\gamma}2^{4D} \notag \\
    & \leq O(1)\frac{1}{\gamma}2^{4D} (D+1)c^D e^{\sqrt{D\ln{1/\epsilon_3}}} 
\notag \\
    & = O(1)\frac{1}{\gamma}2^{4D} (D+1)c^D e^{\sqrt{D\ln \left(c \cdot c_2 \frac{(D+1)^{2.75}}{\gamma}2^{4D}\right)}} \notag \\
    &\leq 2^{O(D)}
\end{align}
We specify now how we are performing the $n$ applications of the lemma, in terms of the choice of $\tilde{m}$ and the random variables.
Let $\tilde{B}:=\lceil c_1  \frac{d}{\epsilon_3^2}\mathfrak{C}^2(q,\epsilon_{3})\rceil$. We use the fact that for large enough constant $c$, $m\geq d\left(\frac{n}{\epsilon}\right)^{c/\gamma} \geq n \tilde{B}$.
For $i=1,\cdots, n-1$ we apply the lemma \ref{lem:approx_indiv_components} with $\tilde{m}=\lfloor\frac{m}{n}\rfloor$ and for $i=n$ with $\tilde{m}=m-(n-1)\lfloor\frac{m}{n}\rfloor$. Also, for the application of the lemma for the $i^{\text{th}}$ datapoint, we use as $U_r^{(0)}$ the $\sqrt{m} W_{(i-1)\lfloor\frac{m}{n}\rfloor+r}^{(0)}$, as $\beta_r^{(0)}$ the $\sqrt{m}b_{(i-1)\lfloor\frac{m}{n}\rfloor+r}^{(0)}$ and as $\alpha_r^{(0)}$ the $a_{(i-1)\lfloor\frac{m}{n}\rfloor+r}^{(0)}$. We apply a union bound and we have that with probability at least $1-n\exp(-\Omega(\sqrt{m/n}))=1-\exp(-\Omega(\sqrt{m/n}))$, from the $n$ applications of the lemma, we get these $\Delta W^{(i)}$ and we construct $\Delta W=\left[\Delta W^{(1)},\cdots,\Delta W^{(n)}\right]\in \mathbb{R}^{d\times m}$ and we have that
 \begin{align*}
     \|\Delta W\|_{2,\infty}\leq O\left(m^{1/3}\frac{\mathfrak{C}\left(q,\epsilon_3\right)}{\lfloor\frac{m}{n}\rfloor}\right)\leq O\left(\frac{n\ \mathfrak{C}\left(q,\epsilon_3\right)}{m^{2/3}}\right)\leq \frac{\left(n/\epsilon\right)^{O(\gamma^{-1})}}{m^{2/3}}
 \end{align*}
 and 
 \begin{align*}
   \forall x\in \mathcal{X},~~
  \left| \sum_{r=1}^{m} a_r^{(0)}\langle \Delta W_r,x \rangle \ind\{\langle W_r^{(0)},x \rangle +b_r^{(0)} \geq 0\}-\sum_{i=1}^n y_i q(\langle x_i,x \rangle)\right| \leq n\epsilon_3\leq \epsilon/3
\end{align*}
where the last inequality is a crude bound, but sufficient for our purposes.
\qed
\\
\\

We proceed with the proof of Lemma \ref{lem:approx_indiv_components}
 \begin{proof}
     We apply Lemma \ref{lem:indicator_to_function} using $\phi(z)=y_iq(z)$ and $\epsilon_1=\epsilon_3$. Observe that since $|y_i|\leq O(1)$, the complexities of $\phi$ and $q$ are the same, up to constants. Thus, we have that  there exists a function $h:\R^2\rightarrow\left[- \mathfrak{C}\left(q,\epsilon_3\right), \mathfrak{C}\left(q,\epsilon_3\right)\right]$ such that
 
\begin{align}\label{eq:app_indic_to_function}
\forall x\in {\cal X}, ~~
    \Big|\E_{u \sim {\cal N}(0,I_d), \beta \sim {\cal N}(0,1) } \left[\ind\{ \langle u,x \rangle +\beta\geq0 \} \  h( \langle x_i , u \rangle, \beta)\right] 
   -y_iq( \langle x_i,x \rangle ) \Big|\leq \epsilon_3
\end{align}
Now, we fix an $x\in \cal{X}$. From Hoeffding's inequality, we get that with probability at least $1-\exp \left(-\Omega \left(\frac{\tilde{m}\epsilon_3^2}{ \mathfrak{C}^2(q,\epsilon_3)}\right)\right)$,

\begin{align*}
    \Bigg|\frac{1}{\tilde{m}}\sum_{r=1}^{\tilde{m}}  \ind\{\langle U_r^{(0)},x \rangle +\beta_r^{(0)} \geq 0 \} h \left
   ( \langle x_i , U_r^{(0)} \rangle, \beta_r^{(0)} \right) - \E_{u \sim {\cal N}(0,I_d), \beta \sim {\cal N}(0,1) } \left[\ind\{ \langle u,x \rangle +\beta\geq0 \} \  h( \langle x_i , u \rangle, \beta)\right] \Bigg| 
   \leq \epsilon_3
\end{align*}
By setting $\Delta W_r^{(i)}=\frac{1}{\alpha_r^{(0)}}\frac{2h \left( \langle x_i , U_r^{(0)} \rangle, \beta_r^{(0)} \right)}{\tilde{m}}e_d$ (where $e_d=(0,0,\dots,0,1
)\in \mathbb{R}^d$) we have that $\| \Delta W^{(i)} \|_{2,\infty}\leq O\left(m^{1/3}\frac{\mathfrak{C}\left(q,\epsilon_3\right)}{\tilde{m}}\right)$ and since $x_d=1/2$ for all $x\in \cal{X}$, we have that for every $x\in \cal{X}$, with probability at least $1-\exp \left(-\Omega \left(\frac{\tilde{m}\epsilon_3^2}{ \mathfrak{C}^2(q,\epsilon_3)}\right)\right)$,
\begin{align}\label{eq:concentration}
    \Bigg|\sum_{r=1}^{\tilde{m}}   \ind\{\langle U_r^{(0)},x \rangle +\beta_r^{(0)} \geq 0 \}\alpha_r^{(0)} \langle  \Delta W_r^{(i)},x \rangle -\E_{u \sim {\cal N}(0,I_d), \beta \sim {\cal N}(0,1) } \left[\ind\{ \langle u,x \rangle +\beta\geq0 \} \  h( \langle x_i , u \rangle, \beta)\right] \Bigg| 
   \leq \epsilon_3
\end{align}
The fact that \ref{eq:concentration} holds with overwhelming probability, enables us to take a union bound over a fine-grained net of $\cal{X}$. Let $c > 0$ be a sufficiently large constant (e.g. 10) and let ${\cal X}_1$ be a maximal $\frac{1}{\tilde{m}^c}$-net of ${\cal X}$. It is well-known that $|{\cal X}_1|\leq\left(\frac{1}{\tilde{m}}\right)^{O(d)}$. By applying a union bound over ${\cal X}_1$  for \ref{eq:concentration}, we have that for $\tilde{m}\geq c_1  \frac{d}{\epsilon_3^2} \mathfrak{C}^2(q,\epsilon_3))$ ($c_1$ is a large constant),
\begin{align}\label{eq:concentration_on_net}
   & \Pr \Bigg[\forall x\in \mathcal{X}_1,~~ 
  \Bigg|\sum_{r=1}^{\tilde{m}}   \ind\{\langle U_r^{(0)},x \rangle +\beta_r^{(0)} \geq 0 \}\alpha_r^{(0)} \langle  \Delta W_r^{(i)},x \rangle \\
   &~~~~~~~~~~~~~~~~~~~~~-\E_{u \sim {\cal N}(0,I_d), \beta \sim {\cal N}(0,1) } \left[\ind\{ \langle u,x \rangle +\beta\geq0 \} \  h( \langle x_i , u \rangle, \beta)\right] \Bigg| >\epsilon_3 \Bigg] \notag \\
   \leq & ~ \exp\left(O(d\log m)\right)\exp\left(-\Omega \left(\frac{\tilde{m}\epsilon_3^2}{ \mathfrak{C}^2(q,\epsilon_3)}\right)\right)\\
   = & ~ \exp\left(-\Omega \left(\frac{\tilde{m}\epsilon_3^2}{ \mathfrak{C}^2(q,\epsilon_3)}\right)\right)
\end{align}

The final step is to show that with overwhelming probability, for all $x\in {\cal X}_1$, if we perturb $x$ by at most $\frac{1}{\tilde{m}^c}$ in $\ell_2$, then the LHS of \ref{eq:concentration} changes very slightly. Because $c$ can be chosen to be as large constant as we want, this "stability" requirement is very mild and also straightforward to prove. We proceed with a formal proof. 

 We will show the stability property for a fixed $x\in \cal{X}$ and then we will do a union bound. Let $v\in \mathbb{R}^d$ such that $x+v\in \cal{X}$ and $\|v\|_2\leq \frac{1}{\tilde{m}^c}$. This $v$ can be arbitrarily correlated with the randomness $\{U_r^{(0)},\beta_r^{(0)},\alpha_r^{(0)}\}_{r=1}^{\tilde{m}}$. We will show the following claim
 
 \begin{claim}\label{cl:stability}
   For all $x\in \mathcal{X}_1$, with probability at least $1-\exp(-\Omega(\sqrt{\tilde{m}}))$,
 
 \begin{align*}
   D_1:= & ~ \Bigg|\sum_{r=1}^{\tilde{m}}   \ind\{\langle U_r^{(0)},x+v \rangle +\beta_r^{(0)} \geq 0 \}\alpha_r^{(0)} \langle  \Delta W_r^{(i)},x+v \rangle -\sum_{r=1}^{\tilde{m}}   \ind\{\langle U_r^{(0)},x \rangle +\beta_r^{(0)} \geq 0 \}\alpha_r^{(0)} \langle  \Delta W_r^{(i)},x \rangle  \Bigg| \notag \\
   \leq & ~  O\left(\frac{\mathfrak{C}\left(q,\epsilon_3\right)}{\sqrt{\tilde{m}}}\right)
\end{align*}
and
\begin{align*}
   D_2:= & ~ \Big| \E_{u \sim {\cal N}(0,I_d), \beta \sim {\cal N}(0,1) } \left[\ind\{ \langle u,x+v \rangle +\beta\geq0 \} \  h( \langle x_i , u \rangle, \beta)\right] \\
   &~~~~~~~~~~~~~~~~~~~~~~~~~~~~-\E_{u \sim {\cal N}(0,I_d), \beta \sim {\cal N}(0,1) } \left[\ind\{ \langle u,x \rangle +\beta\geq0 \} \  h( \langle x_i , u \rangle, \beta)\right] \Big| \notag \\
   \leq & ~ O\left(\frac{\mathfrak{C}\left(q,\epsilon_3\right)}{\sqrt{\tilde{m}}}\right)
\end{align*}  
 \end{claim}
 With this claim at hand we can finish the proof of the Lemma \ref{lem:approx_indiv_components}. Indeed, combining \ref{eq:app_indic_to_function}, \ref{eq:concentration_on_net} and the above claim, we have that with probability at least $1-\exp\left(-\Omega \left(\frac{\tilde{m}\epsilon_3^2}{ \mathfrak{C}^2(q,\epsilon_3)}\right)\right)-\exp(-\Omega(\sqrt{\tilde{m}}))$,
 \begin{align*}
&\forall x\in \cal{X}, \notag\\
   & \left| \sum_{r=1}^{\tilde{m}} \alpha_r^{(0)}\langle \Delta W_r^{(i)},x \rangle \ind\{\langle U_r^{(0)},x \rangle +\beta_r^{(0)} \geq 0\}-y_iq(\langle x_i,x \rangle)\right| \leq  O\left(\frac{\mathfrak{C}\left(q,\epsilon_3\right)}{\sqrt{\tilde{m}}}\right)+ 2\epsilon_3 
\end{align*}
since $\tilde{m}\geq c_1  \frac{d}{\epsilon_3^2} \mathfrak{C}^2(q,\epsilon_3)$ for a large constant $c_1$, we are done.
\end{proof}
It remains to prove the Claim \ref{cl:stability}.
 \begin{proof}
We start with bounding $D_1$. Observe that from the way we constructed $\Delta W^{(i)}$, we have that for $j\leq d-1$, $\Delta W_{rj}^{(i)}=0$. At the same time, $v_d=0$, so $\langle  \Delta W_r^{(i)},v \rangle=0$. Using that  $\| \Delta W^{(i)} \|_{2,\infty}\leq O\left(m^{1/3}\frac{\mathfrak{C}\left(q,\epsilon_3\right)}{\tilde{m}}\right)$ and $|\alpha_r^{(0)}|=\frac{1}{m^{1/3}}$, we get that 

 \begin{align*}
   D_1 \leq & ~ O\left(\frac{\mathfrak{C}\left(q,\epsilon_3\right)}{\tilde{m}}\right) \sum_{r=1}^{\tilde{m}}  \Big| \ind\{\langle U_r^{(0)},x+v \rangle +\beta_r^{(0)} \geq 0 \}-   \ind\{\langle U_r^{(0)},x \rangle +\beta_r^{(0)} \geq 0 \} \Big|  \notag \\
   \leq & ~ O\left(\frac{\mathfrak{C}\left(q,\epsilon_3\right)}{\tilde{m}}\right) \sum_{r=1}^{\tilde{m}}   \ind \left\{\sgn(\langle U_r^{(0)},x+v \rangle +\beta_r^{(0)} ) \neq \sgn (\langle U_r^{(0)},x \rangle +\beta_r^{(0)} )\right\}  \notag \\
   \leq & ~ O\left(\frac{\mathfrak{C}\left(q,\epsilon_3\right)}{\tilde{m}}\right) \sum_{r=1}^{\tilde{m}}   \left(\ind \left\{|\langle U_r^{(0)},x \rangle +\beta_r^{(0)}|\leq \frac{1}{\sqrt{\tilde{m}}} \right\}+ \ind \left\{ \|U_r^{(0)}\|_2>c_2\sqrt{\tilde{m}}\right\} \right) .
\end{align*}
 where $c_2$ can be chosen to be as large as we want (but still a constant) as long as we choose the constant $c$, that appears at the construction of the net, to be sufficiently large. We prove the following claim, whose proof is almost identical to the proof of Claim \ref{cl:bound_U_0}, but we provide it for completeness.
 \begin{claim}\label{cl:bound_U_0}
     With probability at least $1-\exp(-\Omega( \tilde{m}))$, for all $r\in[\tilde{m}]$, $\|U_r^{(0)}\|_{2}\leq O(\sqrt{\tilde{m}})$.
   \end{claim}
   \begin{proof}
   From concentration of sum of independent Chi-Square random variables, we have that for all $r$, with probability at least $1-\exp(-\Omega( \tilde{m}^2/d))$, $\|U_r^{(0)}\|_2^2\leq O(\tilde{m})$. Since $\tilde{m}\geq d$, a union bound over all $r$ finishes the proof of the claim.
   \end{proof}
   
 Thus, by appropriately choosing $c_2$, we get that with probability at 
least $1-\exp(-\Omega(\tilde{m}))$,
 \begin{align*}
   &D_1\leq   O\left(\frac{\mathfrak{C}\left(q,\epsilon_3\right)}{\tilde{m}}\right) \sum_{r=1}^{\tilde{m}}   \ind \left\{|\langle U_r^{(0)},x \rangle +\beta_r^{(0)}|\leq \frac{1}{\sqrt{\tilde{m}}} \right\}  
\end{align*}
Now, $\ind \left\{|\langle U_r^{(0)},x \rangle +\beta_r^{(0)}|\leq \frac{1}{\sqrt{\tilde{m}}} \right\}$ are $\tilde{m}$ independent Bernoulli random variables and because of Claim \ref{cl:anti-concentration}, the corresponding probability is at most $O\left(\frac{1}{\sqrt{\tilde{m}}}\right)$. Thus, from Chernoff bounds we get that with probability at least $1-\exp(-\Omega(\sqrt{\tilde{m}}))$, $\sum_{r=1}^{\tilde{m}}   \ind \left\{|\langle U_r^{(0)},x \rangle +\beta_r^{(0)}|\leq \frac{1}{\sqrt{\tilde{m}}} \right\}\leq O(\sqrt{\tilde{m}})$. By applying a union bound, we get that with probability at least $1-\exp(-\Omega(\sqrt{\tilde{m}}))-\exp(-\Omega(\tilde{m}))=1-\exp(-\Omega(\sqrt{\tilde{m}}))$, $D_1\leq   O\left(\frac{\mathfrak{C}\left(q,\epsilon_3\right)}{\sqrt{\tilde{m}}}\right)$.
\\
We proceed with bounding $D_2$. Since $|h(\cdot)|\leq\mathfrak{C}\left(q,\epsilon_3\right)$, we have
\begin{align*}
   & D_2\leq \mathfrak{C}\left(q,\epsilon_3\right) \E_{u \sim {\cal N}(0,I_d), \beta \sim {\cal N}(0,1) } \left[ \left | \ind\{ \langle u,x+v \rangle +\beta\geq0 \} -\ind\{ \langle u,x \rangle +\beta\geq0 \} \right| \right] \\
   &\leq \mathfrak{C}\left(q,\epsilon_3\right) \E_{u \sim {\cal N}(0,I_d), \beta \sim {\cal N}(0,1) } \left[ \ind\{ |\langle u,x \rangle +\beta|\leq \frac{1}{\sqrt{\tilde{m}}} \} +\ind\{ \|u\|_2>c_2 \sqrt{\tilde{m}}\}  \right] 
\end{align*}
where $c_2$ is the same constant as before.
But, same as before, $\Pr_{u \sim {\cal N}(0,I_d), \beta \sim {\cal N}(0,1) } \left[|\langle u,x \rangle +\beta|\leq \frac{1}{\sqrt{\tilde{m}}} \right]\leq O(\frac{1}{\sqrt{\tilde{m}}})$ and 
$\Pr_{u \sim {\cal N}(0,I_d), \beta \sim {\cal N}(0,1) } \left[\|u\|_2>c_2 \sqrt{\tilde{m}}\right]\leq \exp(-\Omega(\tilde{m}))$. So, $D_2\leq   O\left(\frac{\mathfrak{C}\left(q,\epsilon_3\right)}{\sqrt{\tilde{m}}}\right)$.
\end{proof}







\end{document}